\documentclass[11pt,a4paper]{scrartcl}

\usepackage[table]{xcolor}
\usepackage{calc}
\usepackage{amssymb}
\usepackage{amstext}
\usepackage{amsmath}
\usepackage{amsfonts}
\usepackage{amscd}
 \usepackage{amsthm}
\usepackage{hyperref}
\usepackage{multirow}
 
 \definecolor{light-gray}{gray}{0.95}
 \definecolor{lightlight-gray}{gray}{0.98}
  \usepackage{epsfig}

\usepackage{graphicx}
\usepackage{caption}
\usepackage{subcaption}

\newtheorem{theorem}{Theorem}[section]
\newtheorem{proposition}{Proposition}[section]

\newtheorem{corollary}{Corollary}[section]
\newtheorem{example}{Example}[section]
\newtheorem{remark}{Remark}[section]
 \newtheorem{definition}{Definition}[section]

 \newcommand\N{\mathbb{N}}

\newcommand\Q{\mathbb{Q}}
\newcommand\R{\mathbb{R}}

\usepackage[margin=0.65in]{geometry}

\begin{document}
\title{A general unified framework for interval pairwise comparison matrices\footnote{This is a preprint of the paper ``Cavallo B. and Brunelli M., A general unified framework for interval pairwise comparison matrices. \emph{International Journal of Approximate Reasoning}, 93, 178--198. \htmladdnormallink{DOI:10.1016/j.ijar.2017.11.002.}{http://dx.doi.org/10.1016/j.ijar.2017.11.002}}}
\author{
{Bice Cavallo$^a$ and Matteo Brunelli$^b$}
\\
{\normalsize $^a$Department of Architecture, University of Naples ``Federico II'', Italy} \\
{\normalsize e--mail:
\texttt{bice.cavallo@unina.it}}\\
{\normalsize $^b$Department of Industrial Engineering, University of Trento, Italy} \\
{\normalsize e--mail:
\texttt{matteo.brunelli@unitn.it}}}

\date{}

\maketitle \thispagestyle{empty}

\begin{center}
{\textbf Abstract }
\end{center}

{\small \noindent 
Interval Pairwise Comparison Matrices have been widely used to account for uncertain statements concerning the preferences of decision makers.
Several approaches have been proposed in the literature, such as multiplicative and fuzzy interval matrices. In this paper, we propose a general unified approach  to Interval Pairwise Comparison Matrices, based on Abelian linearly ordered groups. In this framework, we generalize some consistency conditions provided for multiplicative and/or fuzzy interval pairwise comparison matrices and provide inclusion relations between them. Then, we provide a concept of distance between intervals that, together with a notion of mean defined over real continuous Abelian linearly ordered groups, allows us to provide a consistency index and an indeterminacy index.
In this way, by means of suitable isomorphisms between Abelian linearly ordered groups, we will be able to compare the inconsistency and the indeterminacy of different kinds of Interval Pairwise Comparison Matrices, e.g. multiplicative, additive, and fuzzy, on a unique Cartesian coordinate system..}

 \vspace{0.3cm}
 \vspace{0.3cm}

\section{Introduction} \label{sec:Introduction}
As their name suggests, Pairwise Comparison Matrices (PCMs) have been a long standing technique for comparing alternatives and their role has been pivotal in the development of modern decision making methods. In accordance with decision theory, in this paper we shall consider a finite non-empty set of $n$ entities (e.g. criteria or alternatives) $X=\{ x_{1},\ldots,x_{n} \}$, and the object of our investigation is the set of comparisons between them with respect to one of their properties. That is, we are interested in the subjective estimations $a_{ij}~\forall i,j \in \{ 1,\ldots,n \}$, where $a_{ij}$ is a numerical representation of the intensity of preference of $x_{i}$ over $x_{j}$.
 
With respect to the values that $a_{ij}$ can assume  and their interpretation, it is fundamental to be aware that various proposals have been presented, studied, and applied in the literature to solve real-world problems. 

The foremost type of representation of valued preferences, at least with respect to the number of real-world applications is probably the multiplicative representation, used among others by Saaty in the theory of the Analytic Hierarchy Process (AHP). In this sense, pairwise comparisons are expressed as positive real numbers, $a_{ij} \in ]0 , +\infty[$ satisfying a condition of multiplicative reciprocity, $a_{ij} \cdot a_{ji} = 1$. We shall note that the AHP \cite{Saaty77} is not the only method using this scheme for pairwise comparisons. For instance, proponents of Multi Attribute Value Theory (MAVT) such as Keeney and Raiffa \cite{KeeneyRaiffa1976} and Belton and Stewart \cite{BeltonStewart2002} advocate the use of pairwise comparisons to estimate the ratios between weights of criteria when the value function is additive. Hereafter, representations of preference of this kind will be called \emph{multiplicative}.

Reciprocal preference relations, whose origins can be traced back at least to a study by Zermelo \cite{Zermelo1929}, assume that intensities of preferences are represented on the open unit interval; that is, $a_{ij}\in ]0,1[$. Similarly to the multiplicative case, reciprocal preference relations obey a condition of reciprocity, in this case $a_{ij} + a_{ji} = 1$. Interestingly, such a representation was studied, among others, also by Luce and Suppes \cite{LuceSuppes1965} under the name of `probabilistic preference relations' and has been widely popularized within the fuzzy sets community under the name of `fuzzy preference relations'. Instances of influential studies on these mathematical structures under the fuzzy lens have been offered by Tanino \cite{Tanino1984}, Herrera-Viedma et al. \cite{ViedmaHerreraChiclanaLuque} and Kacprzyk \cite{Kacprzyk1986}. For sake of simplicity, and because the unit interval recalls the idea of membership function, we shall refer to this case as the \emph{fuzzy} case in the rest of this manuscript.

The third representation considered in this paper shall be called \emph{additive} due to the fact that intensities of preferences are expressed as real numbers, $a_{ij} \in ]-\infty , + \infty[$ and comply with a condition of additive reciprocity, i.e. $a_{ij}+a_{ji}=0$. We shall note that this representation coincides with the Skew-symmetric additive representation of utilities proposed by Fishburn \cite{Fishburn1999} and with the representation used by some decision analysis methodologies such as REMBRANDT \cite{Olson1995}.

All in all, it emerges a picture where the technique of pairwise comparisons plays an important role within decision theory. Moreover, in spite of their different formulations and interpretations, it was formalized that different representations share the same algebraic structure \cite{CavalloDapuzzo}, based on Abelian linearly ordered groups, i.e. commutative groups equipped with an ordering relation. Hence, to derive results which are general enough to pertain to each of these representations of preferences, we will focus on this more general algebraic representation  and exploit the full potential of group theory. Several authors have already adopted this approach based on Abelian linearly ordered groups (e.g. \cite{Hou,Koczkodaj201681,Ramik2015236,Chen}). Nevertheless, in spite of the general formulation of our results, examples involving specific representations of preferences will be used in the rest of this paper.

More specifically, within this framework, we shall investigate the case of Interval Pairwise Comparisons Matrices (IPCMs) according to which comparison values are expressed as intervals $\tilde{a}_{ij}=[a_{ij}^{-},a_{ij}^{+}] \subset \mathbb{R}$ instead of real numbers. The approach with intervals has been widely used to account for uncertain statements concerning the preferences of a decision maker  (e.g. \cite{Landes20141301,Zhang20141787}) and studied separately in the case of multiplicative preference relations \cite{SaatyVargas1987,SaloHamalainen1995} and in the case of fuzzy preference relations \cite{Xu2004}, just to cite few examples. In this paper, we shall generalize it and derive broader results. More specifically, we will generalize interval arithmetic and propose a concept of metric on intervals when these are subsets of Abelian linearly ordered groups. This will be instrumental to formulate the concept of IPCM and study, in a more general context, the notions of reciprocity, consistency, and indeterminacy. Having done this, we will propose and justify a consistency index which, in concert with an indeterminacy index, can be used to evaluate the acceptability of IPCMs.

There are further papers in the literature that take into consideration consistency and indeterminacy: Wang \cite{Wang2015252} considered multiplicative IPCMs and proposed a geometric mean based uncertainty index to capture the inconsistency in the original multiplicative PCM;  Liu \cite{Liu20092686} measured the consistency of a multiplicative IPCM by computing Saaty's consistency index \cite{Saaty77} of one or two  associated PCMs; Li et al. \cite{Li2016628} and  Wang and Chen  \cite{WANG201459}  proposed as indeterminacy index the geometric mean of the ratios $\frac{a_{ij}^+}{a_{ij}^-} $ for  both  multiplicative and fuzzy IPCMs.
 However, no paper proposed a consistency index to be computed directly from the IPCM, i.e. without considering associated PCMs, and no paper proposes both a general consistency index and a general indeterminacy index suitable for each kind of IPCM (e.g. additive, multiplicative, and fuzzy).
%

The paper is organized as follows. Section \ref{sec:notation} provides the necessary notions and notation for the real-valued case. Next, in Section \ref{sec:Interval_arithmetics}, we discuss the idea of intervals defined over a special type of group structure. By drawing from the previous two sections, in Section \ref{sec:IPCMs} we present a general notion of interval pairwise comparison matrix which has the merit of unifying different approaches under the same umbrella. This will give us the possibility, in Section \ref{sec:consistency}, to discuss reciprocity and consistency conditions in a more general setting. In Sections \ref{sec:inconsistency} and \ref{sec:indeterminacy}, we introduce a consistency and an indeterminacy index, respectively. These indices can be used in concert to evaluate the acceptability of preferences.  Section \ref{sec:conclusions} draws some conclusions and proposes directions for future work. 
Finally,  Appendix contains the proofs of the statements
%

\section{Notation and preliminaries}
\label{sec:notation}
In this section, we will provide notation and preliminaries
  which will be necessary in the rest of the paper.

\subsection{Abelian linearly ordered groups} \label{sec:alo_groups}
We start providing  definitions  and  essential notation about Abelian linearly ordered groups in order  to define Pairwise Comparison Matrices (Subsection \ref{sec:PCM_over_Alo_groups}) and Interval Pairwise Comparison Matrices (Sections \ref{sec:IPCMs} and \ref{sec:consistency});  for further details the reader can refer to \cite{CavalloDapuzzo}.

\begin{definition}\label{Def_alo_group}
 Let $G$ be a non-empty set, $\odot: G\times G \rightarrow G$ a binary operation on G, $ \leq$ a   weak order on $G$.
  Then, $\mathcal{G}= (G, \odot, \leq)$ is an  \emph{Abelian linearly ordered group}, \emph{Alo-group} for short, if 
 $(G, \odot)$ is an Abelian  group and
\begin{equation}
\label{ordercons} a\leq b\Rightarrow  a\odot c \leq b\odot c.
\end{equation}
\end{definition}

Let us denote with $e$ the  \emph{identity} with respect to $\odot$,  $a^{(-1)}$ the
\emph{inverse} of $a\in G$ with respect  to $\odot$ and $\div$ the operation defined by $a \div b= a \odot b^{(-1)}\quad \forall a,b \in G$.
 Then, we have \cite{CavalloDapuzzo}: 
   \begin{equation}
\label{composition-1}
 a^{(-1)}=\textit{e}\div a,\quad (a \odot b)^{(-1)}=a ^{(-1)}\odot b ^{(-1)},  \quad (a \div b) ^{(-1)}= b \div
 a ,\quad a \geq e \Leftrightarrow a^{(-1)} \leq e, \quad a\leq b \Leftrightarrow b^{(-1)} \leq a^{(-1)}.
\end{equation}
Furthermore, we can define the concept of ($n$)-\emph{natural-power}.
 \begin{definition}\label{eq:def_n_power}\cite{CavalloDapuzzo}
Let $\mathcal{G}= (G, \odot, \leq)$  be an Alo-group and $n\in \mathbb{N}_{0}$. The ($n$)-\emph{natural-power} $a^{(n)}$ of $a\in G$ is
defined as follows:
$$
a^{(n)}= \left\{
            \begin{array}{ll}
              e,    \quad &\text{if } n=0 \\
             a^{(n-1)}\odot a,\quad & \text{if } n\geq
  1.
            \end{array}
          \right.
$$
\end{definition}
%

 Let $z\in \mathbb{Z}$; then the ($z$)-\emph{integer- power} $a^{(z)}$ of $a\in G$ is defined as follows \cite{CavalloDapuzzo}:
\begin{equation}     \label{eq:def_z_power}
a^{(z)}= \left\{
            \begin{array}{ll}
              a^{(n)}, \quad & \quad if \quad z=n\in \mathbb{N}_{0}\\
               (a^{(n)})^{(-1)}&  \quad if\quad  z=-n, \quad n\in \mathbb{N}.
            \end{array}
          \right.
\end{equation}

 An \emph{isomorphism } between two Alo-groups $\mathcal{G}= (G,
\odot, \leq)$ and $\mathcal{H}= (H, *, \leq)$ is a bijection $\phi:G\rightarrow H$
 that  is   both a lattice isomorphism and a group isomorphism, that is:
\begin{equation} 
\label{isoeq} a<b \Leftrightarrow \phi(a)< \phi(b)\quad and\quad \phi(a\odot
b) = \phi(a) * \phi(b);
\end{equation}
where
  $<$ is the strict simple order defined by \lq\lq$a<b \Leftrightarrow a\leq b \; and\; a\neq b$\rq\rq.


\subsubsection{ $\mathcal{G}$-norm and  $\mathcal{G}$-distance}
 \label{sec:distanceAloGrous}
%
By definition, an Alo-group   $\mathcal{G}$ is a
\emph{lattice ordered group} \cite{Birkoff}. Namely, there exists $ \max\{a, b \}$, for each  $a, b \in  G$. Thus, the existence of the max value  between two elements of $G$ and the existence of the inverse of each element of $G$ allows us to formulate the notions of $\mathcal{G}$-norm and $\mathcal{G}$-distance, which are generalizations to $\mathcal{G}$ of the usual concepts of norm and distance.

\begin{definition} \cite{CavalloDapuzzo} \label{norm}
Let $\mathcal{G}= (G, \odot, \leq)$  be an Alo-group. Then, the
function:
\begin{equation} || \cdot||_{\mathcal{G}} : a\in G
\rightarrow  ||a||_{\mathcal{G}} =  \max\{a, a^{(-1)} \}\in G
\end{equation}
is a $\mathcal{G}$-\emph{norm}, or a \emph{norm} on $\mathcal{G}$.
\end{definition}

\begin{proposition}\label{triangle}\cite{CavalloDapuzzo} The $\mathcal{G}$-norm satisfies the properties:
\begin{enumerate}
  \item $||a||_\mathcal{G}=||a^{(-1)}||_\mathcal{G}$;
  \item $ a\leq ||a||_\mathcal{G}$;
  \item $ ||a||_\mathcal{G}\geq e$;
   \item $||a||_\mathcal{G}= e\Leftrightarrow a=e$;
\item  $||a\odot b||_\mathcal{G}\leq  ||a||_\mathcal{G}\odot ||b||_\mathcal{G}.$
\end{enumerate}
\end{proposition}

\begin{definition}\label{dis}\cite{CavalloDapuzzo}
 Let $\mathcal{G}= (G, \odot, \leq)$  be an Alo-group.
 Then, the operation
 \begin{displaymath}
\label{d} d:(a,b)\in G \times G\rightarrow d(a, b)\in G
\end{displaymath}
is a $\mathcal{G}$-\emph{metric} or  $\mathcal{G}$-\emph{distance} if:
 \begin{enumerate}
  \item $d(a, b)\geq e$;
  \item  $d(a, b)= e  \Leftrightarrow a=b$;
   \item  $d(a, b)=d(b, a)$;
  \item  $d(a, b) \leq d(a, c) \odot d(c, b)$.
\end{enumerate}
\end{definition}

\begin{proposition}\cite{CavalloDapuzzo} \label{dcirc}
Let $\mathcal{G}= (G, \odot, \leq)$  be an Alo-group. Then, the
operation
 \begin{equation}
d_{\mathcal{G}}:(a,b)\in G \times G \rightarrow d_{\mathcal{G}}(a, b)=||a\div b||_{\mathcal{G}}\in G
\end{equation}
 is a  $\mathcal{G}$-distance.
 \end{proposition}

%

Let  $\phi$ be an isomorphism between $\mathcal{G}= (G,
\odot, \leq)$ and $\mathcal{H}= (H, *, \leq)$, $g_1,g_2 \in G$ and $h_1,h_2\in H$; then, Cavallo and D'Apuzzo \cite{CavalloDapuzzo} prove that:
\begin{equation}
\label{eqisodis} d_{\mathcal{H}}(h_1, h_2) =
\phi(d_{\mathcal{G}}(\phi^{-1}(h_1), \phi^{-1}(h_2))), \quad
d_{\mathcal{G}}(g_1,g_2) = \phi^{-1} (d_{\mathcal{H}}(\phi(g_1), \phi(g_2))).
\end{equation}

\subsubsection{$\mathcal{G}$-mean in real continuous Alo-groups}
An Alo-group $\mathcal{G}= (G, \odot, \leq)$ is called \emph{continuous} if the operation $\odot$ is continuous \cite{CavalloDapuzzo}, and \emph{real} if  $G$ is a subset of the real line $\R$ and  $\leq$ is the weak order  on $G$ inherited from the  usual order on  $\R$. From now on, we will assume that $\mathcal{G}= (G, \odot, \leq)$ is a real continuous Alo-group, with $G$ an open interval.
%
%
%
%
%
%
%
%
%
Under these assumptions, 
the equation $x^{(n)}=a$ has a unique solution \cite{CavalloDapuzzo}; thus,  it is reasonable to consider the following notions of ($n$)-\emph{root} and  $\mathcal{G}$-\emph{mean}.
 \begin{definition}\cite{CavalloDapuzzo}For each $n\in \mathbb{N}$ and $a \in G$,
the ($n$)-\emph{root} of $a$, denoted by $a^{(\frac{1}{n})}$,  is the unique solution of the equation $x^{(n)}=a$, that is:
\begin{equation*}
\label{eq:n_root}\left( a^{(\frac{1}{n})} \right)^{(n)}=a.
\end{equation*}
\end{definition}
 \begin{definition}\cite{CavalloDapuzzo}\label{def:G_mean}
The $\mathcal{G}$-\emph{mean} $ m_{\mathcal{G}}( a_{1},a_{2}, ..., a_{n})$ of
the  elements $a_{1},a_{2}, ..., a_{n}$  of $ G$ is
\begin{equation}
\label{mean} m_{\mathcal{G}}( a_{1},a_{2}, ..., a_{n})=\begin{cases}
   a_{1} & \text{for n=1 }, \\
      \left( \bigodot^{n}_{i=1} a_{i} \right)^{(1/n)}& \text {for $n\geq
      2$}.\nonumber
\end{cases}
\end{equation}
\end{definition}
%


For each $q=\frac{m}{n}\in \Q$, with $m \in \mathbb{Z}$ and $n \in \mathbb{N}$,
and for each $a \in G$, the $(q)$-\emph{rational-power} $a^{(q)}$ is defined as follows  \cite{CavalloDapuzzoSquillanteIJIS3}:
\begin{equation}\label{def_a_el_q}
a^{(q)}= (a^{(m)})^{(\frac{1}{n})}.
\end{equation}

The following are examples of real continuous Alo-groups which will be relevant in the rest of the paper (see \cite{CavalloDapuzzoIJIS2,CavalloDapuzzoSquillanteIJIS3} for details):
\begin{description}

 \item [\textbf{Multiplicative Alo-group.}]  $\mathcal{R}^+=(\R^+, \cdot, \leq)$, where  $\R^+=]0, + \infty[$ and  $\cdot$ is
  the usual multiplication on $\R$. Thus,  the $\mathcal{R^+}$-mean operator is the geometric mean,
$$ m_{\mathcal{R^+}}\left( a_{1}, ..., a_{n} \right) =\left( \prod _{i=1}^{n}
 a_{i} \right)^{\frac{1}{n}}$$
and the $\mathcal{R^+}$-distance between $a$ and $b$ is $d_{\mathcal{R}^+}(a,b)= \max\left\{\frac{a}{b},\frac{b}{a} \right\} $.
  \item[\textbf{Additive Alo-group.}]$  \mathcal{R} = (\R, +, \leq)$, where  $\R=]-\infty, + \infty[$ and $+$ is  the usual addition   on $\R$. Thus,  the $\mathcal{\mathcal{R}}$-mean operator is the arithmetic mean,
$$ m_{\mathcal{R}}(
a_{1},..., a_{n})=\dfrac{\sum
_{i=1}^{n}a_{i}}{n}$$
and the $\mathcal{\mathcal{R}}$-distance between $a$ and $b$ is $d_{\mathcal{R}}(a,b)= \max \{a - b,b - a \} = |a-b|$.
%
     \item [\textbf{Fuzzy Alo-group.}]  $\mathcal{I}=(I, \otimes, \leq)$, where $I=]0,1[$ and $\otimes: ]0,1[^{2}\rightarrow ]0,1[$
     is the operation defined by
\begin{equation}\label{eq:fuzzy_operation}
a\otimes b=\frac{ab}{ab+(1-a)(1-b)}.
\end{equation}     
Thus, the $\mathcal{\mathcal{I}}$-mean operator is given by the following function \cite{CavalloDapuzzoSquillanteIJIS3}:
\begin{equation}
\label{eq:fuzzy_mean}
m_{\mathcal{I}}( a_{1}, ...,  a_{n})=\frac{\sqrt[n]{\prod_{i=1}^{n}a_{i}}}{\sqrt[n]{\prod_{i=1}^{n}a_{i}}
+ \sqrt[n]{\prod_{i=1}^{n}(1-a_{i})}}.
\end{equation}
The operation $\otimes$ is the restriction to $]0,1[^{2}$ of a widely known uninorm introduced by Yager and Rybalov \cite{YagerRybalov} and then studied by Fodor et al.\frenchspacing \cite{FodorYagerRybalov}. For this Alo-group, the $\mathcal{\mathcal{I}}$-distance between $a$ and $b$ is the following one:
\[
d_{\mathcal{I}}(a,b)= \max \left\{ \frac{a(1-b)}{a(1-b)+(1-a)b} , \frac{b(1-a)}{b(1-a)+(1-b)a}  \right\}.
\]

%
%
\end{description}

It was proven that for each pair $\mathcal{G}= (G, \odot, \leq)$ and $\mathcal{H}= (H, *, \leq)$ of real continuous Alo-groups with $G$ and $H$ open intervals, there exists an isomorphism between them
\cite{CavalloDapuzzo}.
 For example, the function
%
\begin{equation}
\label{eq:isomorphism_psi_01}
    h: x\in ]0, +\infty[ \mapsto \frac{x}{1+x}\in ]0,1[
\end{equation}
is an isomorphism between multiplicative Alo-group and fuzzy Alo-group.
Another example is the function
\begin{equation}
\label{eq:isomorphism_psi_02}
    g: x\in ] -\infty, +\infty[ \mapsto \frac{e^x}{1+e^x}\in ]0,1[,
\end{equation}
which is an isomorphism between the additive Alo-group and the fuzzy Alo-group.
\\ Moreover, let  $\phi$ be an isomorphism between $\mathcal{G}= (G,
\odot, \leq)$ and $\mathcal{H}= (H, *, \leq)$, $g_1,g_2 \ldots, g_n \in G$ and $h_1,h_2 \ldots, h_n \in H$; then, Cavallo and D'Apuzzo \cite{CavalloDapuzzo} prove that:
\begin{eqnarray}
 m_{\mathcal{G}}( g_{1}, g_{2}, ..., g_{n})  & = & \phi^{-1}\big(m_{\mathcal{H}}(\phi(g_{1}), \phi(g_{2}), ...,\phi(g_{n})) \big);  \nonumber\\
\label{eq:mean_isomorphism}  \\ m_{\mathcal{H}}( h_{1}, h_{2}, ..., h_{n}) & =
&\phi\big(m_{\mathcal{G}}(\phi^{-1}(h_{1}), \phi^{-1}(h_{2}), ...,\phi^{-1}(h_{n}))
\big). \nonumber \end{eqnarray}

We believe that it is important to stress that the use of Alo-groups and the definition of group isomorphisms between them is not a mere theoretical exercise. Alo-groups and their isomorphisms are \emph{necessary} to show the formal equivalence between different approaches. For instance, in his widely used textbook, Fraleigh \cite{Fraleigh2002} defines an isomorphism as \lq\lq the concept of two systems being structurally identical, that is, one being just like the other except for names\rq\rq.

\subsection{Pairwise Comparison Matrices over a real continuous Alo-group} \label{sec:PCM_over_Alo_groups}
 Quantitative pairwise comparisons are a  useful tool  for estimating  the relative weights  on a set  $X=\{x_{1},x_{2},..., x_{n}\}$ of  decision elements  such as  criteria or alternatives.  
 Pairwise comparisons can be modelled by a quantitative preference relation on $X$:  
$$\mathcal{A}:
   (x_{i}, x_{j})\in X\times X\rightarrow a_{ij}=\mathcal{A}(x_{i}, x_{j})\in G$$
where  $G$ is an open interval of $\R$ and $a_{ij}$ quantifies the preference intensity of $x_{i}$  over  $x_{j}$.  When the cardinality of $X$ is small, $\mathcal{A}$  can be   represented by a \emph{Pairwise Comparison Matrix} (\emph{PCM})
 \begin{equation}\label{pcm}
A = \; \bordermatrix{
~     & x_{1}  & x_{2}  & \cdots & x_{n} \cr
x_{1} & a_{11} & a_{12} & \dots  & a_{1n} \cr
x_{2} & a_{21} & a_{22} & \cdots & a_{2n} \cr
\vdots & \vdots & \vdots & \ddots & \vdots \cr
x_{n} & a_{n1} & a_{n2} & \cdots &  a_{nn}
}.
 \end{equation}

 \begin{definition}\label{def:dot_reciprocity}\cite{CavalloDapuzzo}
 A PCM $A=(a_{ij})$  is a  $\mathcal{G}$-\emph{reciprocal} if it verifies the condition:
  \begin{displaymath}
a_{ji} =a_{ij}^{(-1)} \quad \forall \; i,j\in \{1, \ldots, n\}.
\end{displaymath}
\end{definition}

Let $A=(a_{ij})$ be a $\mathcal{G}$-reciprocal PCM and $ (\sigma(1), \ldots, \sigma(n))$ a permutation of $(1,\ldots,n)$; then, by  Definition \ref{def:dot_reciprocity},  for each permutation  $\sigma$,  the following equalities hold true:
$$ a_{\sigma(j)\sigma(i)}=a_{\sigma(i)\sigma(j)}^{(-1)}~ \forall \; i,j\in \{1, \ldots, n\}$$ and, as a consequence,  $A^\sigma$ defined as follows:

\begin{equation}\label{eq:sigma_PCM}
   A^\sigma=\left(\begin{array}{cccc}
a_{\sigma(1)\sigma(1)} & a_{\sigma(1)\sigma(2)} & \cdots &
a_{\sigma(1)\sigma(n)} \\
a_{\sigma(2)\sigma(1)} & a_{\sigma(2)\sigma(2)} & \cdots &
a_{\sigma(2)\sigma(n)} \\
\vdots & \vdots & \ddots & \vdots \\
a_{\sigma(n)\sigma(1)} & a_{\sigma(n)\sigma(2)} & \cdots &
a_{\sigma(n)\sigma(n)}\end{array}\right) 
\end{equation}
is  a $\mathcal{G}$-reciprocal PCM too. 
 In other words, if we apply row-column permutations to  a $\mathcal{G}$-reciprocal PCM, then also the resulting matrix will be a $\mathcal{G}$-reciprocal PCM.



 \begin{definition}\label{def:dot_consistency}\cite{CavalloDapuzzo}
 $A=(a_{ij})$ is a $\mathcal{G}$-\emph{consistent} PCM, if   verifies the following condition:
\begin{equation}\label{eq:consistency}
a_{ik}= a_{ij}\odot a_{jk}\qquad\forall i, j, k \in \{1, \ldots, n\}.
\end{equation}
\end{definition}

\begin{proposition}\cite{CavalloDapuzzo}\label{prop:consistency_i<j<k}
Let  $A=(a_{ij})$  be $\mathcal{G}$-reciprocal PCM. Then, the following statements are equivalent:
 \begin{enumerate}
\item $A=(a_{ij})$ is a $\mathcal{G}$-consistent PCM;
\item $a_{ik}= a_{ij}\odot a_{jk}\qquad\forall i< j<k \in \{1, \ldots, n\}.$
\end{enumerate}

\end{proposition}


 \begin{definition}\label{def:consistency_index} \cite{CavalloDapuzzo}   Let $A=(a_{ij})$ be a $\mathcal{G}$-reciprocal PCM of order  $n \geq3$. Then, its  $\mathcal{G}$-consistency index is:
     $$
 I_{\mathcal{G}}(A)=
     \left( \bigodot_{i<j<k}
d_{\mathcal{G}}(a_{ik}, a_{ij}\odot
	a_{jk})\right)^{(\frac{1}{|T|})},
$$
with $T=\{(i,j,k):i<j<k\}$ and   $|T|=\frac{n(n-2)(n-1)}{6}$ its cardinality.
 \end{definition}

We stress that,  in  Definition \ref{def:consistency_index}, $|T| \in \N$, with $|T|\geq 1$, and  the $\mathcal{G}$-consistency index is a $\mathcal{G}$-mean (see Definition \ref{def:G_mean}) of $|T|$ $\mathcal{G}$-distances from  $\mathcal{G}$-consistency.
Moreover, 
let  $\phi$ be an isomorphism between $\mathcal{G}= (G,
\odot, \leq)$ and $\mathcal{H}= (H, *, \leq)$, $A'=\phi(A)=(\phi(a_{ij}))$; then, Cavallo and D'Apuzzo \cite{CavalloDapuzzo} prove that:
\begin{equation}
I_{\mathcal{H}}(A')=\phi(I_{\mathcal{G}}(A)).
\end{equation}

\begin{proposition}\cite{CavalloDapuzzo}\label{prop:unique_element_for_consistency}
Let  $A=(a_{ij})$  be $\mathcal{G}$-reciprocal PCM.  Then, the following statements hold:
\begin{enumerate}
\item $I_{\mathcal{G}}(A) \geq e$;
\item $I_{\mathcal{G}}(A) =e   \Leftrightarrow A$ is $\mathcal{G}$-consistent;
\item $I_{\mathcal{G}}(A)=I_{\mathcal{G}}(A^{\sigma})$ for all permutations $\sigma$.
\end{enumerate}
\end{proposition}



\section{Intervals over a real continuous Alo-group} \label{sec:Interval_arithmetics}
In this section, by respecting standard rules of interval arithmetic \cite{Cabrer20141623,RAMON2009},
 we extend interval arithmetic to work on a real continuous Alo-group $\mathcal{G}=(G,\odot,\leq)$, with $G$ an open interval of $\R$.
 For notational convenience, let  $[G]$ be the set of closed intervals over $G$, that is:
\begin{equation}\label{eq:interval_on_G}
 [G]=\{\tilde{a}=[a^-,a^+]|  a^-,a^+ \in G, \; a^- \leq a^+\}.
\end{equation}
The subset of all \textit{singleton intervals (points)} is denoted by
$[G]_p$, that is:
\begin{equation}
[G]_p=\{\tilde{a}=[a^-,a^+] \in  [G]|   a^- =a^+\}.
\end{equation}
Of course, if $ \tilde{a} \in [G]_p$ then $\tilde{a} $ degenerates in an element of $G$.
Equality relation on $[G]$ is defined as follows:
\begin{equation} \label{eq:equality_intervals}
\tilde{a}=\tilde{b} \Leftrightarrow  a^-=b^- \text{ and } a^+=b^+.
\end{equation}
Following \cite{RAMON2009} and \cite{Dawood}, for each $\tilde{a}=[a^-,a^+] \in   [G]$, we denote with:
\begin{equation}\label{eq:reciprocalInterval}
\tilde{a}^{(-1)}=[(a^+)^{(-1)},(a^-)^{(-1)}]
\end{equation}    
the reciprocal interval of  $\tilde{a}$; of course, $\tilde{a}^{(-1)} \in   [G]$  because, by the last equivalence in \eqref{composition-1}, $(a^+)^{(-1)} \leq (a^-)^{(-1)}$.\\
Let us consider 
 $\tilde{a}=[a^{-},a^{+}]$ and $ \tilde{b}=[b^{-},b^{+}] \in   [G]$;
then we can borrow the definition of binary operation on intervals and set: 
\begin{equation}\label{eq:interval_operation}
 \tilde{a} \odot_ { [G]} \tilde{b}=\{ a \odot b |\, a \in \tilde{a}, b \in \tilde{b}\}
\end{equation}
and consequently
\begin{equation} \label{eq:inverseOperationIntervals}
\tilde{a} \div_ { [G]}  \tilde{b}= \tilde{a} \odot_ { [G]}  \tilde{b}^{(-1)}.
\end{equation}


%
The following theorem provides a further representations of $\tilde{a} \odot_ { [G]}  \tilde{b}$ and $\tilde{a} \div _ { [G]} \tilde{b}$. Its main scope is that of providing closed forms for the operations $\odot_{ [G]}$ and $\div_{ [G]}$. This will help simplify the operations and derive results in closed form.
\begin{theorem} \label{Theorem_productIntervals}
Let  $\tilde{a}, \tilde{b} \in  [G] $; then, the following equalities hold:
\begin{align*}
\tilde{a} \odot_ { [G]}  \tilde{b}&=[a^- \odot b^-, a^+ \odot b^+],\\
\tilde{a} \div _ { [G]} \tilde{b}&=[a^- \div b^+, a^+ \div b^-].
\end{align*}
\end{theorem}
\begin{proposition}\label{Prop:noInverse}
The following assertions hold:
\begin{enumerate}
\item $[e,e]\in   [G]$ is the identity  with respect to $\odot_ {[G]}$;
\item  $\tilde{a} \in   [G]$ has inverse  with respect to $\odot_ {[G]}$ if and only if $\tilde{a} \in  [G]_p$.
\end{enumerate}
\end{proposition}

\begin{example}
From the previous proposition we know that  if $\tilde{a} \not \in   [G]_p$, then $\tilde{a}^{(-1)}$ in \eqref{eq:reciprocalInterval} is not its inverse. 
\\ Let us consider the multiplicative Alo-group; then, e.g. we have $[2,4] \odot_ {[\mathcal{R}^{+}]}  [1/4,1/2] =[1/2,2]  \neq [1,1]$. \\
Let us consider the additive Alo-group; then, e.g. we have $[2,4] \odot_ {[\mathcal{R}]} [-4,-2] =[-2,2] \neq [0,0] $. \\
Let us consider the fuzzy Alo-group; then, e.g. we have $[0.6,0.7] \odot_ {[\mathcal{I}]} [0.3,0.4] =[0.39, 0.61] \neq [0.5,0.5] $. 
\end{example}

A strict partial order on $[G]$ is defined as follows:
\begin{equation} \label{eq:partial_order_intervals}
\tilde{a}  <_ {[G] }\tilde{b} \Leftrightarrow a^+<b^-;
\end{equation}
thus, we set 
\begin{equation} \label{eq:partial_weak_order_intervals}
\tilde{a}  \leq_ {[G] }\tilde{b} \Leftrightarrow \tilde{a}=\tilde{b} \text{ or } \tilde{a}  <_ {[G] }\tilde{b}.
\end{equation}

It is important to note that, as one should expect, the real case is just an instance of the interval-valued case when the intervals are singletons. Hence, all the results obtained in the interval-valued case are compatible with, and apply to, the real valued case as well. For sake of precision, the following theorem stipulates this connection in the form of a isomorphism between Alo-groups.

\begin{theorem} \label{Theorem:isomorphism_intevals_Alo_group_G}
$[\mathcal{G}]_p= ([G]_p, \odot_ {[{G}]}, \leq_{[G]} )$ is an Alo-group isomorphic to $\mathcal{G}=(G, \odot, \leq)$.
\end{theorem}

Since the proof of the previous theorem implicitly shows that $\odot_ {[G]}$ is a monoid operation for the Abelian monoid $[\mathcal{G}]=(  [G], \odot_ {[G]} )$, from now on, we will use $\odot_ {[\mathcal{G}]}$ instead of $\odot_ {[G]}$.
\\
\\
The distance between two real numbers in an Alo-group was already defined by Cavallo and D'Apuzzo \cite{CavalloDapuzzo} and here recalled in Proposition \ref{dcirc}. Now we shall extend this proposal to the more general case of intervals. First, we propose and study a notion of $[\mathcal{G}]$-norm, that is the generalization to intervals of the concept of $\mathcal{G}$-norm in Definition \ref{norm}.

\begin{definition} \label{norm_intervals}
The $[\mathcal{G}]$-\emph{norm} on $[\mathcal{G}]$ is given by the following function:
\begin{equation*}  || \cdot||_{[\mathcal{G}]} : \tilde{a}\in [G]
\rightarrow  || \tilde{a}||_{[\mathcal{G}]} = \max\{||a^-||_\mathcal{G}, ||a^+||_\mathcal{G}\} \in G.
\end{equation*}
\end{definition}

Similarly to Proposition \ref{triangle}, we provide the following properties of  $[\mathcal{G}]$-norm:
\begin{proposition}\label{Prop:propertiesNorm} The $[\mathcal{G}]$-norm satisfies the following properties:
\begin{enumerate}
  \item $||\tilde{a}||_{[\mathcal{G}]}=||\tilde{a}^{(-1)}||_{[\mathcal{G}]}$;
  \item $ a^-, a^+\leq ||\tilde{a}||_{[\mathcal{G}]}$;
  \item $ ||\tilde{a}||_{[\mathcal{G}]}\geq e$;
   \item $||\tilde{a}||_{[\mathcal{G}]}= e\Leftrightarrow a^-=a^+=e$;
\item  $||\tilde{a}\odot_ {[\mathcal{G}]} \tilde{b}||_{[\mathcal{G}]} \leq  ||\tilde{a}||_{[\mathcal{G}]}\odot ||\tilde{b}||_{[\mathcal{G}]}.$
\end{enumerate}
\end{proposition}

We are now ready to extend the concept of $\mathcal{G}$-distance to the interval-valued case and we call it $[\mathcal{G}]$-distance.

\begin{definition}\label{def:dis_intervals}
The  function 
 \begin{displaymath}
m:(\tilde{a},\tilde{b}) \in [G] \times [G] \rightarrow m(\tilde{a},\tilde{b}) \in G
\end{displaymath}
is a $[\mathcal{G}]$-\emph{metric} or  $[\mathcal{G}]$-\emph{distance} if:
 \begin{enumerate}
	\item $m(\tilde{a},\tilde{b}) \geq e$;
	\item $m(\tilde{a},\tilde{b}) = e \Leftrightarrow \tilde{a}=\tilde{b}$:
	\item $m(\tilde{a},\tilde{b}) = m(\tilde{b},\tilde{a})$;
	\item $m(\tilde{a},\tilde{b}) \leq m(\tilde{a},\tilde{c}) \odot m(\tilde{c},\tilde{b})$.
\end{enumerate}
\end{definition}
With the following proposition, we introduce a function acting as a $[\mathcal{G}]$-distance.
%
\begin{proposition}\label{prop:dis_intervals}
 The
function
 \begin{equation*}
d_{[\mathcal{G}]}:(\tilde{a}, \tilde{b})\in  [G] \times [G]\rightarrow d_{[\mathcal{G}]}(\tilde{a}, \tilde{b})=  ||[a^{-} \div b^{-},a^{+} \div b^{+}]||_{[\mathcal{G}]}\in G
\end{equation*}
 is a  $[\mathcal{G}]$-distance.
 \end{proposition}

\begin{remark}
As one should expect, for the additive Alo-group,  $d_{[\mathcal{G}]}$ collapses into the distance between real intervals, i.e.  $d_{[\mathcal{R}]}(\tilde{a}, \tilde{b})=\max \{ |a^{-}-a^{+}|,|b^{-}-b^{+}| \} $.
\end{remark}

\section{Interval pairwise comparison matrices over a real continuous Alo-group} 
\label{sec:IPCMs}

Let us assume that $\mathcal{G}= (G, \odot, \leq)$ is a real continuous Alo-group, with $G$ an open interval, and  $X=\{x_{1},x_{2},..., x_{n}\}$  a set
of  decision elements  such as  criteria or alternatives.\\
Having laid down the necessary mathematical foundations, we are now ready to formalize and study the concept of interval pairwise comparison matrix, where each entry is an interval in $G$ (i.e. an element of $[G]$).
 Let us consider the following quantitative preference relation on $X$:  
$$\mathcal{\tilde{A}}:
   (x_{i}, x_{j})\in X\times X\rightarrow \tilde{a}_{ij}=\mathcal{\tilde{A}}(x_{i}, x_{j})\in [G]$$
where the \emph{interval} $\tilde{a}_{ij}=[a_{ij}^{-},a_{ij}^{+}] \subset G$  represents the uncertain estimation of the preference intensity of $x_{i}$  over  $x_{j}$.  When the cardinality of $X$ is small, $\mathcal{\tilde{A}}$  can be   represented by an \emph{Interval Pairwise Comparison Matrix} (\emph{IPCM})
 \begin{equation}\label{ipcm}
 \tilde{A} = \; \bordermatrix{
~     & x_{1}  & x_{2}  & \cdots & x_{n} \cr
x_{1} & \tilde{a}_{11} & \tilde{a}_{12} & \dots  & \tilde{a}_{1n} \cr
x_{2} & \tilde{a}_{21} & \tilde{a}_{22} & \cdots & \tilde{a}_{2n} \cr
\vdots & \vdots & \vdots & \ddots & \vdots \cr
x_{n} & \tilde{a}_{n1} & \tilde{a}_{n2} & \cdots &  \tilde{a}_{nn}
}
=
\begin{pmatrix}
[{a}_{11}^{-},{a}_{11}^{+}] & [{a}_{12}^{-},{a}_{12}^{+}] & \dots  & [{a}_{1n}^{-},{a}_{1n}^{+}] \\
 [{a}_{21}^{-},{a}_{21}^{+}] & [{a}_{22}^{-},{a}_{22}^{+}] & \cdots & [{a}_{2n}^{-},{a}_{2n}^{+}] \\
\vdots & \vdots & \ddots & \vdots \\
[{a}_{n1}^{-},{a}_{n1}^{+}] & [{a}_{n2}^{-},{a}_{n2}^{+}] & \cdots &  [{a}_{nn}^{-},{a}_{nn}^{+}]
\end{pmatrix}.
 \end{equation} 

Let  $\tilde{A}=(\tilde{a}_{ij})$  be a IPCM; we say that  $\tilde{A}=(\tilde{a}_{ij})$ degenerates  in a PCM over $\mathcal{G}= (G, \odot, \leq)$ if  $ \tilde{a}_{ij} \in [G]_p, \forall i,j \in \{1, \ldots, n\}$.

\subsection{$[\mathcal{G}]$-reciprocal IPCMs} 

As it was done for PCMs, we can now formulate and study the concept of reciprocity for IPCMs in a more general framework.

\begin{definition}\label{def:alternativa3}
  $\tilde{A}=(\tilde{a}_{ij})$, with $\tilde{a}_{ij} \in [G]$ for each $i,j =1, \ldots, n$,  is a $[\mathcal{G}]$-reciprocal   IPCM if:
\begin{equation}
\label{eq:recip}
\tilde{a}_{ji}=\tilde{a}_{ij}^{(-1)}.
\end{equation}
\end{definition}

\begin{corollary}\label{cor:alternativa3}
  $\tilde{A}=(\tilde{a}_{ij})$, with $\tilde{a}_{ij} \in [G]$ for each $i,j =1, \ldots, n$,  is a $[\mathcal{G}]$-reciprocal   IPCM if and only if
\begin{equation}
\label{eq:recip2}
a_{ij}^- \odot a_{ji}^+=a_{ij}^+ \odot a_{ji}^-=e.
\end{equation}
\end{corollary}

The following examples will show that  $[\mathcal{G}]$-reciprocity  is  suitable for the three most widely used types of IPCMs.

\begin{example} \label{example:multiplicativeIntervalPCM} 
The following IPCM
$$\tilde{A}=\left(
  \begin{array}{ccc}
    [1,1] &[\frac{1}{4}, \frac{1}{2}] & [6,7]  \\
    \; [2,4]  & [1,1] & [3,5] \\
        \;[\frac{1}{7}, \frac{1}{6}] &     \;[\frac{1}{5}, \frac{1}{3}] & [1,1]\\
  \end{array}
\right) $$
is a multiplicative $[\mathcal{R}^+]$-reciprocal  IPCM;  thus, each entry is an interval in $\R^+$.
\end{example}

\begin{example} \label{Example:IntervalAdditivePCM}
The following IPCM
$$\tilde{A}=\left(
  \begin{array}{cccc}
    [0,0] & [4,7] &[2,4]   & \\
    \;[-7,-4] & [0,0]  & [-3,-2]   &\\
        \;[-4,-2] &     \;[2,3] & [0,0] \\
  \end{array}
\right) $$
is an additive  $[\mathcal{R}]$-reciprocal   IPCM;  thus, each entry is an interval in $\R$.
\end{example}

\begin{example}
The following IPCM
$$\tilde{A}=\left(
  \begin{array}{ccc}
    [0.5,0.5] & [0.6,0.7] & [0.7,0.8]  \\
    \;[0.3, 0.4] &  [0.5,0.5]   & [0.6,0.8] \\
        \;[0.2,0.3] &     \;[0.2, 0.4] & [0.5,0.5]  \\
  \end{array}
\right) $$
is a fuzzy  $[\mathcal{I}]$-reciprocal  IPCM; thus, each entry is an interval in $I=]0,1[$.
\end{example}

\begin{proposition}\label{Prop:IPCMdegenratesPCM}
$\tilde{A}=(\tilde{a}_{ij})$ is a $[\mathcal{G}]$-reciprocal IPCM with 
\begin{equation}
\label{eq:cond}
\tilde{a}_{ij} \odot_{[\mathcal{G}]} \tilde{a}_{ji} =[e,e] ~~\forall i,j
\end{equation}
if and only if   $\tilde{A}=(\tilde{a}_{ij})$ degenerates in a $\mathcal{G}$-reciprocal PCM.
\end{proposition}


From now on, we assume that $\tilde{A}$ is a $[\mathcal{G}]$-reciprocal IPCM.
Let $ (\sigma(1), \ldots, \sigma(n))$ be  a permutation of $(1,\ldots,n)$; then, similarly to $A^{\sigma}$ in \eqref{eq:sigma_PCM}, we define $ \tilde{A} ^\sigma$ as follows:
\begin{equation}\label{eq:permutation_Pi_grande} \tilde{A} ^\sigma =\left(\begin{array}{cccc}\tilde{a}_{\sigma(1)\sigma(1)} & \tilde{a}_{\sigma(1)\sigma(2)} & ... &
\tilde{a}_{\sigma(1)\sigma(n)}
 \\\tilde{a}_{\sigma(2)\sigma(1)} & \tilde{a}_{\sigma(2)\sigma(2)} & ... &
\tilde{a}_{\sigma(2)\sigma(n)} \\... & ... & ... & ... \\\tilde{a}_{\sigma(n)\sigma(1)} & \tilde{a}_{\sigma(n)\sigma(2)} & ... &
\tilde{a}_{\sigma(n)\sigma(n)}\end{array}\right) .
\end{equation}

By using an argument similar to the one used to show that $\mathcal{G}$-reciprocity of  $A$ guarantees the $\mathcal{G}$-reciprocity of $A^{\sigma}$, we provide the following proposition, which extends the result to IPCMs:

\begin{proposition}\label{Prop:reciprocalIPCM_permutations}
$\tilde{A}=(\tilde{a}_{ij})$ is $[\mathcal{G}]$-reciprocal if and only if $\tilde{A}^{\sigma}$ is $[\mathcal{G}]$-reciprocal for all permutations $\sigma$.
\end{proposition}


\section{$[\mathcal{G}]$-Consistency conditions for  $[\mathcal{G}]$-reciprocal IPCMs}
\label{sec:consistency}

The formulation of consistency conditions and reliable indices to estimate the extent of their violation have played a pivotal role in the development of the theory of pairwise comparisons. As emerges from recent studies \cite{Brunelli2016b}, there is not a meeting of minds on the best way of capturing inconsistency. This is even more so in the case of IPCMs since in this context there is not even a uniquely accepted condition of consistency.

In this section, we generalize to $[\mathcal{G}]$-reciprocal IPCMs  consistency conditions which were initially proposed in the literature for 
fuzzy IPCMs \cite{WangLi2012,WANG2015890} and multiplicative IPCMs \cite{Li2016628,Wang2015}.

\subsection{Liu's $[\mathcal{G}]$-consistency}
 Let $\tilde{A}=([a_{ij}^-, a_{ij}^+])$ be a $[\mathcal{G}]$-reciprocal IPCM, then we define $L=(l_{ij})$ and $R=(r_{ij})$ as follows
\begin{equation}\label{eq:L_R}
l_{ij}=\begin{cases}
  a_{ij}^- & \;  i<j \\
  e & \;  i=j \\
  a_{ij}^+ & \;  i>j
\end{cases}  \;\;\;  \;\;\;  \;\;\;r_{ij}=\begin{cases}
  a_{ij}^+ & \;  i<j \\
  e & \;  i=j \\
  a_{ij}^- & \;  i>j
\end{cases}\;.
\end{equation}

We stress that $L=(l_{ij})$ and $R=(r_{ij})$ are $\mathcal{G}$-reciprocal PCMs.  At this point, we can state the first condition of consistency, which we call Liu's $[\mathcal{G}]$-consistency because it generalizes a consistency condition provided by Liu \cite{Liu20092686} for multiplicative IPCMs.
\begin{definition} \label{def:LiuodotConsistency}
$\tilde{A}=([a_{ij}^-, a_{ij}^+])$ is Liu's $[\mathcal{G}]$-consistent if
\begin{equation}\label{eq:LiuodotConsistency}
\begin{cases}
l_{ik}=l_{ij} \odot l_{jk} \\
r_{ik}=r_{ij} \odot r_{jk}
\end{cases}\quad \forall i,j,k \in \{1, \ldots, n\}; 
\end{equation}
that is,  $L$ and $R$  are $\mathcal{G}$-consistent PCMs over 
$(G, \odot, \leq)$.
\end{definition}

\begin{proposition}\label{prop:LiuEquivalent}
The following statements are equivalent:
\begin{enumerate}
\item $\tilde{A}=([a_{ij}^-, a_{ij}^+])$ is Liu's $[\mathcal{G}]$-consistent;
\item $
\begin{cases}
l_{ik}=l_{ij} \odot l_{jk} \\
r_{ik}=r_{ij} \odot r_{jk}
\end{cases}
\;\;
\forall i<j<k;
$
\item 
$
\tilde{a}_{ik} = \tilde{a}_{ij} \odot_ {[\mathcal{G}]} \tilde{a}_{jk} \vspace{0.3cm}\quad  \forall i<j<k.
$
\end{enumerate}
\end{proposition}

It is crucial to notice  that, analogously to Liu's consistency  defined for multiplicative IPCMs (see \cite{LiuPedryczZhang}),  Liu's $[\mathcal{G}]$-consistency is not invariant with respect to the permutation of alternatives.  This means that Liu's $[\mathcal{G}]$-consistency depends on the labeling of criteria$/$alternatives and, as a consequence,  Liu's $[\mathcal{G}]$-consistency (inconsistency) of $\tilde{A}$ does not imply   Liu's $[\mathcal{G}]$-consistency (inconsistency) of $\tilde{A}^{\sigma}$ for different permutations $\sigma$.

\begin{example}\label{example:Liu_not_preserves_consistency}
Let us consider the following two additive $[\mathcal{R}]$-reciprocal IPCMs:
\[
\tilde{A}=
\begin{pmatrix}
 [0,0] & [2,4] & [4,7] \\
 [-4,-2] & [0,0] & [2,3] \\
 [-7,-4] & [-3,-2] & [0,0]
\end{pmatrix}
\; \; \; \;
\tilde{A}^{\sigma}=
\begin{pmatrix}
 [0,0] &  [-4,-2] & [2,3] \\
 [2,4] & [0,0] &[4,7] \\
 [-3,-2] & [-7,-4] & [0,0]
\end{pmatrix},
\]
where the latter is obtained by applying $\sigma= \{ 2,1,3 \}$ to the former. Permutation $\sigma$ does not change the structure of the preferences, yet only the first IPCM is deemed Liu $[\mathcal{R}]$-consistent.
\end{example}

The soundness of such consistency condition has thus been questioned in recent papers \cite{MengEtAl2017,Wang2015,MengTan2017} as it seems to violate a principle according to which a simple reordering of alternatives which leaves the preferences unchanged shall not affect the consistency of these latter ones \cite{BrunelliFedrizzi2015}. Consequently, to overcome this issue, more recent formulations of consistency conditions for IPCMs are invariant under permutations of alternatives.

\subsection{Approximate $[\mathcal{G}]$-consistency} 
In order to deal with the above mentioned shortcoming of Liu's  $[\mathcal{G}]$-consistency, Liu et al. \cite{LiuPedryczZhang} proposed an approximate consistency condition for multiplicative IPCMs. This consistency condition has also been used in applications of multiplicative IPCM to solve problems of partner selection \cite{Liu2016partner}.
In this section, we provide the notion of approximate $[\mathcal{G}]$-consistency to generalize approximate consistency. 

Let  $\tilde{A}=([a_{ij}^-, a_{ij}^+])$ be a $[\mathcal{G}]$-reciprocal IPCM, $\sigma$  a permutation of $\{1,2, \ldots, n\}$, $\tilde{A}^\sigma$  in \eqref{eq:permutation_Pi_grande} and   $L^\sigma=(l_{ij}^\sigma)$ and $R^\sigma=(r_{ij}^\sigma)$ with $l_{ij}^\sigma$  and $r_{ij}^\sigma$  defined as follows:

\begin{equation}\label{eq:L_R_sigma}
l_{ij}^\sigma=\begin{cases}
  a_{\sigma(i)\sigma(j)}^- & \;  i<j \\
  e & \;  i=j\\
  a_{\sigma(i)\sigma(j)}^+ & \;  i>j
\end{cases}  \;\;\;  \;\;\;  \;\;\;r_{ij}^\sigma=\begin{cases}
  a_{\sigma(i)\sigma(j)}^+ & \;  i<j\\
  e & \;  i=j \\
  a_{\sigma(i)\sigma(j)}^- & \;  i>j
\end{cases}\quad.
\end{equation}


\begin{example}
Let us consider the  multiplicative $[\mathcal{R^+}]$-reciprocal  IPCM  in Example \ref{example:multiplicativeIntervalPCM}. Let   $\sigma=\{1,2,3\}$, then $\tilde{A}^\sigma=\tilde{A}$ and 
$$L^\sigma=\left(
  \begin{array}{ccc}
    1 & \frac{1}{4} & 6 \\
   \; 4 & 1   & 3\\
       \; \frac{1}{6}&    \frac{1}{3}&1  \\
  \end{array}
\right) \quad  R^\sigma=\left(
  \begin{array}{ccc}
    1 & \frac{1}{2} & 7 \\
   \; 2 & 1   & 5\\
       \; \frac{1}{7}&    \frac{1}{5}&1  \\
  \end{array}
\right). $$
Let  $\sigma_1=\{1,3,2\}$, then
$$\tilde{A}^{\sigma_1}=\left(
  \begin{array}{ccc}
    [1,1] &  [6,7]   &  [\frac{1}{4}, \frac{1}{2}]  \\
     \;[\frac{1}{7}, \frac{1}{6}] & [1,1] & \;[\frac{1}{5}, \frac{1}{3}] \\
        \;[2,4] &   [3,5] & [1,1]\\
  \end{array}
\right) $$
and 
$$L^{\sigma_1}=\left(
  \begin{array}{ccc}
    1 & 6 & \frac{1}{4} \\
   \;\frac{1}{6} & 1   & \frac{1}{5}\\
       \; 4&   5&1  \\
  \end{array}
\right) \quad  R^{\sigma_1}=\left(
  \begin{array}{ccc}
    1 & 7 & \frac{1}{2} \\
   \;\frac{1}{7} & 1   & \frac{1}{3}\\
       \; 2&   3&1  \\
  \end{array}
\right). $$
\end{example}

\begin{definition} \label{Def:ApproximateodotConsistency}
An IPCM $\tilde{A}=([a_{ij}^-, a_{ij}^+])$ is \emph{approximately $[\mathcal{G}]$-consistent} if there is a permutation $\sigma$ such that $L^\sigma=(l_{ij}^\sigma)$ and $R^\sigma=(r_{ij}^\sigma)$ are $\mathcal{G}$-consistent PCMs over $(G, \odot, \leq)$.
\end{definition}

\begin{remark}\label{remark:equivalenceApproximateConsistency}
We stress that, said in other words, $\tilde{A}=([a_{ij}^-, a_{ij}^+])$ is an approximately $[\mathcal{G}]$-consistent IPCM if and only  if there is a permutation $\sigma$ such that $\tilde{A}^\sigma$ is Liu's  $[\mathcal{G}]$-consistent.
\end{remark}

\begin{example}\label{Example:IntervalAdditivePCM_II_part}
Let us consider  the additive  $[\mathcal{R}]$-reciprocal   IPCM  in Example \ref{Example:IntervalAdditivePCM} and the permutation $\sigma=\{1,3,2\}$. Then:
$$\tilde{A}^\sigma=\left(
  \begin{array}{cccc}
    [0,0] & [2,4] &[4,7]   & \\
    \;[-4,-2] & [0,0]  &  [2,3]  &\\
        \;[-7,-4] &     \;[-3,-2] & [0,0] \\
  \end{array}
\right). $$
The additive PCMs
$$L^{\sigma}=\left(
  \begin{array}{ccc}
    0 & 2 & 4\\
   \;-2& 1   & 2\\
       \; -4&   -2&1  \\
  \end{array}
\right) \quad  R^{\sigma}=\left(
  \begin{array}{ccc}
    0 & 4 & 7\\
   \;-4& 1   & 3\\
       \; -7&   -3&1  \\
  \end{array}
\right)$$
are $\mathcal{R}$-consistent; thus, $\tilde{A}$ in Example \ref{Example:IntervalAdditivePCM} is  an approximately $[\mathcal{R}]$-consistent  IPCM.
\end{example}

\begin{theorem}\label{ApproximateConsistencyDegeneration}
The IPCM $\tilde{A}=([a_{ij}^-, a_{ij}^+])$ degenerates in a $\mathcal{G}$-consistent PCM over $(G, \odot, \leq)$ if and only if  $L^\sigma=(l_{ij}^\sigma)$ and $R^\sigma=(r_{ij}^\sigma)$   are $\mathcal{G}$-consistent PCMs over $(G, \odot, \leq)$ for each permutation $\sigma$.
\end{theorem}

\subsection{$[\mathcal{G}]$-consistency} 
In this section, we generalize the  consistency condition employed by Li et al.~\cite{Li2016628} and Zhang \cite{Zhang2016} for multiplicative IPCMs.

\begin{definition} \label{Def:consistent_IPCM}
$\tilde{A}=(\tilde{a}_{ij})$ is a $[\mathcal{G}]$-consistent  IPCM  if 
\begin{equation}
\label{eq:consistent_IPCM}
\tilde{a}_{ij} \odot_ {[\mathcal{G}]} \tilde{a}_{jk}  \odot_ {[\mathcal{G}]} \tilde{a}_{ki}=  \tilde{a}_{ik} \odot_ {[\mathcal{G}]} \tilde{a}_{kj}  \odot_ {[\mathcal{G}]} \tilde{a}_{ji} \quad \forall i,j,k \in \{1, \ldots, n\}.
\end{equation}
\end{definition}

By Theorem \ref{Theorem_productIntervals}, $[\mathcal{G}]$-consistency in Definition \ref{Def:consistent_IPCM} is equivalent to:
\begin{equation}\label{eq:Equivalence_odotGConsistency}
\begin{cases}
  {a}_{ij}^{-} \odot {a}_{jk}^{-} \odot {a}_{ki}^{-} = {a}_{ik}^{-} \odot {a}_{kj}^{-} \odot {a}_{ji}^{-}  \quad \forall i,j,k \in \{1, \ldots, n\}\\
 {a}_{ij}^{+} \odot {a}_{jk}^{+} \odot {a}_{ki}^{+} ={a}_{ik}^{+} \odot {a}_{kj}^{+} \odot {a}_{ji}^{+}  \quad \forall i,j,k \in \{1, \ldots, n\}.
\end{cases} 
\end{equation}

%
%
%
The following proposition will show that $[\mathcal{G}]$-consistency is invariant with respect to permutations of alternatives.

\begin{proposition} \label{Prop:invariance_odotG_consistency}
$\tilde{A}=(\tilde{a}_{ij})$ is $[\mathcal{G}]$-consistent if and only if $\tilde{A}^{\sigma}$ is $[\mathcal{G}]$-consistent for all permutations $\sigma$.
\end{proposition}


%

\begin{theorem} \label{Theorem_odotG_consistency_equivalences}
Let $\tilde{A}=(\tilde{a}_{ij})$ be a $[\mathcal{G}]$-reciprocal  IPCM. Then, the following assertions are equivalent:
\begin{enumerate}
\item $\tilde{A}=(\tilde{a}_{ij})$ is a $[\mathcal{G}]$-consistent  IPCM;
\item $
\label{eq:consistency_bis}
a_{ik}^- \odot a_{ik}^+= a_{ij}^- \odot a_{ij}^+ \odot a_{jk}^- \odot a_{jk}^+  \quad \forall i,j,k \in \{1, \ldots, n\};
$
\item $
\label{eq:consistency_bis2}
a_{ik}^- \odot a_{ik}^+= a_{ij}^- \odot a_{ij}^+ \odot a_{jk}^- \odot a_{jk}^+  \quad \forall i<j<k.
$
\end{enumerate}

\end{theorem}

\subsection{Comparisons between consistency conditions}  \label{sec:comparison:consistency}
We are now ready to compare Liu's  $[\mathcal{G}]$-consistency (Definition \ref{def:LiuodotConsistency}), approximate $[\mathcal{G}]$-consistency (Definition \ref{Def:ApproximateodotConsistency}) and $[\mathcal{G}]$-consistency (Defintion \ref{Def:consistent_IPCM}). Under the assumption of $[\mathcal{G}]$-reciprocity, the following proposition holds.

\begin{proposition}\label{prop:implicationsConsistency}
The following statements hold:
\begin{enumerate}
\item  If $\tilde{A}$ is   Liu's  $[\mathcal{G}]$-consistent  then  $\tilde{A}$ is   approximately $[\mathcal{G}]$-consistent;
\item If $\tilde{A}$ is  approximately $[\mathcal{G}]$-consistent then  $\tilde{A}$ is   $[\mathcal{G}]$-consistent.
\end{enumerate}
\end{proposition}


We stress that in the same way as $\mathcal{G}$-consistency implies $\mathcal{G}$-reciprocity \cite{CavalloDapuzzo}, both Liu's  $[\mathcal{G}]$-consistency and approximate  $[\mathcal{G}]$-consistency imply  $[\mathcal{G}]$-reciprocity. Conversely, $[\mathcal{G}]$-consistency does \emph{not} imply $[\mathcal{G}]$-reciprocity. It is sufficient to propose a counterexample in the form of the following  multiplicative IPCM:
\[
\begin{pmatrix}
[1,1] & [2,10] & [6,40] \\
[\frac{1}{5},\frac{1}{2}] & [1,1] & [3,4] \\
[\frac{1}{40},\frac{1}{6}] & [\frac{1}{8} , \frac{1}{3}] & [1,1]
\end{pmatrix},
\]
which is $[\mathcal{R^+}]$-consistent but not $[\mathcal{R^+}]$-reciprocal.

Finally, it can be shown that, unlike approximate $[\mathcal{G}]$-consistency, in which case one needs to seek for a Liu's  $[\mathcal{G}]$-consistent IPCM  $\tilde{A}^\sigma$ to guarantee  approximate $[\mathcal{G}]$-consistency of  $\tilde{A}$ (see Remark \ref{remark:equivalenceApproximateConsistency}),  checking $[\mathcal{G}]$-consistency of  $\tilde{A}$ is more immediate because no permutation $\sigma$ has to be considered (see Proposition \ref{Prop:invariance_odotG_consistency}).

The following examples show that the reverse implications in Proposition \ref{prop:implicationsConsistency} are not true.

\begin{example}
The additive IPCM $\tilde{A}^\sigma$ in Example \ref{example:Liu_not_preserves_consistency} is approximately $[\mathcal{R}]$-consistent  (because $\tilde{A}$ is Liu $[\mathcal{R}]$-consistent) but  not Liu $[\mathcal{R}]$-consistent.
%

\end{example}

\begin{example}
The following additive IPCM is $[\mathcal{R}]$-consistent, since $a_{13}^{-} + a_{13}^{+} = a_{12}^{-} + a_{12}^{+} + a_{23}^{-} + a_{23}^{+}$, but not approximately $[\mathcal{R}]$-consistent. 
\[
\begin{pmatrix}
 [0,0] & [0,1] & [0,1] \\
 [-1,0] & [0,0] & [-2,2] \\
 [-1,0] & [-2,2] & [0,0]
\end{pmatrix}
\]
\end{example}

The findings of Proposition \ref{prop:implicationsConsistency} and the previous counterexamples are summarized in Figure \ref{fig:intersection}.
\begin{figure}[htbp]
	\centering
		\includegraphics[width=0.33\textwidth]{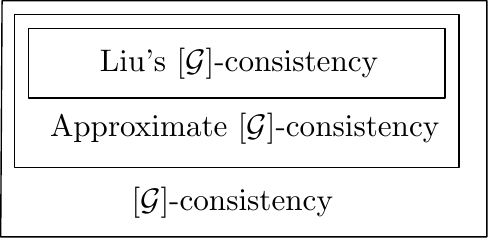}
	\caption{Inclusion relations between consistency conditions.}
	\label{fig:intersection}
\end{figure}

\section{$[\mathcal{G}]$-Consistency index of $[\mathcal{G}]$-reciprocal IPCMs}
\label{sec:inconsistency}
In this section, we propose a method for quantifying the $[\mathcal{G}]$-inconsistency of IPCMs as a violation of the condition of $[\mathcal{G}]$-consistency in Definition \ref{Def:consistent_IPCM}.

Let us denote with $\tilde{a}_{ijk}$ and $\tilde{a}_{ikj}$ the following intervals:
\begin{align*}
\tilde{a}_{ijk} & = [a_{ijk}^-, a_{ijk}^+]=\left[
{a}_{ij}^{-} \odot {a}_{jk}^{-} \odot {a}_{ki}^{-} \, , \,
{a}_{ij}^{+} \odot {a}_{jk}^{+} \odot {a}_{ki}^{+} \right] ,\\
\tilde{a}_{ikj} & = [a_{ikj}^-, a_{ikj}^+]= \left[
{a}_{ik}^{-} \odot {a}_{kj}^{-} \odot {a}_{ji}^{-} \, , \,
{a}_{ik}^{+} \odot {a}_{kj}^{+} \odot {a}_{ji}^{+}
\right] .
\end{align*}

with this notation we can rewrite the $[\mathcal{G}]$-consistency condition in Definition \ref{Def:consistent_IPCM} as follows, 
\begin{equation}
\label{eq:consistency_rewritten}
\tilde{a}_{ijk}= \tilde{a}_{ikj}  \quad \forall i,j,k \in \{1, \ldots, n\} .
\end{equation}
Since $[\mathcal{G}]$-inconsistency manifests itself in the violation of this latter condition and both sides of the equation are intervals, we consider appropriate to quantify $[\mathcal{G}]$-inconsistency by means of a suitable distance  between  $\tilde{a}_{ijk}$ and $\tilde{a}_{ikj}$. At this point, we can employ the $[\mathcal{G}]$-distance $d_{[\mathcal{G}]}$ in Proposition \ref{prop:dis_intervals} to measure the local $[\mathcal{G}]$-inconsistency associated with $i,j,k$ as follows:
\begin{align}\label{eq_oru_consistencyIndex_3}
 d_{[\mathcal{G}]}(\tilde{a}_{ijk},\tilde{a}_{ikj}) = & \max\{ d_\mathcal{G}(\tilde{a}_{ijk}^-, \tilde{a}_{ikj}^-),   d_\mathcal{G}(\tilde{a}_{ijk}^+, \tilde{a}_{ikj}^+) \} .
\end{align}

\begin{proposition}\label{Prop:simplification_DG_aijk_aikj}
Let $\tilde{A}$ be a $[\mathcal{G}]$-reciprocal IPCM. Then, the following equalities hold true:
\begin{align*}
 d_{[\mathcal{G}]}(\tilde{a}_{ijk},\tilde{a}_{ikj})=& \max \left\{  \tilde{a}_{ijk}^{-} \div \tilde{a}_{ikj}^{-} , \; \tilde{a}_{ikj}^{-} \div \tilde{a}_{ijk}^{-} \right\}   =  \max \left\{  \tilde{a}_{ijk}^{-} \div \tilde{a}_{ikj}^{-} , \;    \tilde{a}_{ijk}^{+} \div \tilde{a}_{ikj}^{+}   \right\} =  \\
 =&\max \left\{\tilde{a}_{ikj}^{-} \div \tilde{a}_{ijk}^{-}  , \; \tilde{a}_{ikj}^{+} \div \tilde{a}_{ijk}^{+} \right\} =\max \left\{  \tilde{a}_{ijk}^{+} \div \tilde{a}_{ikj}^{+} , \; \tilde{a}_{ikj}^{+} \div \tilde{a}_{ijk}^{+} \right\} .
\end{align*}
\end{proposition}

%

\begin{example}
Let us consider the multiplicative IPCM  $\tilde{A}$ in Example \ref{example:multiplicativeIntervalPCM};  then, we have
\[
d_{[\mathcal{R^+}]}(\tilde{a}_{123},\tilde{a}_{132})= \max \left\{  \frac{3/28}{12/5} ,  \frac{12/5}{3/28}  \right\}  
= \frac{12/5}{3/28}  = \frac{112}{5} = 22.4.
\]

\end{example}

At this point, similarly  to the ${\mathcal{G}}$-consistency index in Definition \ref{def:consistency_index},  we can extend the local evaluation of the  $[\mathcal{G}]$-inconsistency to an entire IPCM of order $n \geq 3$ thanks to the concept of ${\mathcal{G}}$-mean  (see Definition \ref{def:G_mean}) as follows:

\begin{definition}\label{def:[G]_consistencyIndex}
Let $\tilde{A}$ be a $[\mathcal{G}]$-reciprocal IPCM of order  $n \geq3$. Then, its $[\mathcal{G}]$-consistency index is
\begin{equation}\label{eq_oru_consistencyIndex}
 I_{[\mathcal{G}]} \left( \tilde{A} \right) = \left( \bigodot_{i<j<k} d_{[\mathcal{G}]}(\tilde{a}_{ijk},\tilde{a}_{ikj})  \right)^{(\frac{1}{|T|})}, 
\end{equation}
with $T=\{(i,j,k):i<j<k\}$ and   $|T|=\frac{n(n-2)(n-1)}{6}$ its cardinality.
\end{definition}

Similarly to Definition \ref{def:consistency_index},   we stress that,  in  Definition \ref{def:[G]_consistencyIndex}, $|T| \in \N$, with $|T|\geq 1$, and   the $[\mathcal{G}]$-consistency index is a $\mathcal{G}$-mean (see Definition \ref{def:G_mean}) of $|T|$ $[\mathcal{G}]$-distances from  $[\mathcal{G}]$-consistency. Moreover, 
let  $\phi$ be an isomorphism between $\mathcal{G}= (G,
\odot, \leq)$ and $\mathcal{H}= (H, *, \leq)$, $\tilde{A}'=\phi(\tilde{A})=(\phi(\tilde{a}_{ij}))$; then, by \eqref{eq_oru_consistencyIndex_3}, equivalence in \eqref{isoeq}, \eqref{eqisodis} and \eqref{eq:mean_isomorphism}, we have that:
\begin{equation}
I_{\mathcal{[H]}}(\tilde{A}')=\phi(I_{\mathcal{[G]}}(\tilde{A})).
\end{equation}

\begin{proposition}\cite{CavalloDapuzzo}\label{prop:unique_element_for_[G]_consistency}
Let $\tilde{A}$ be a $[\mathcal{G}]$-reciprocal IPCM; then: 
\begin{equation*}
\label{ eq:IG_>_e}
 I_{\mathcal{[G]}}(\tilde{A}) \geq e, \quad   I_{\mathcal{[G]}}(\tilde{A})  =e   \Leftrightarrow  \tilde{A} \; is \; [\mathcal{G}]\text{-consistent IPCM}.
\end{equation*}
\end{proposition}

%

If all the entries of an IPCM collapse to singletons, then the  $[\mathcal{G}]$-consistency index $I_{[\mathcal{G}]}$ for IPCMs becomes the $\mathcal{G}$-consistency index $I_{\mathcal{G}}$ for PCMs (Definition \ref{def:consistency_index}).
%
Hence, it is important to know that, to corroborate the soundness of $I_{\mathcal{G}}$, it was shown \cite{Brunelli2016,cavallo2012investigating} that it satisfies a set of reasonable properties and therefore, at present, it seems a reasonable function for estimating inconsistency.
Moreover, contrarily to the approaches by Liu \cite{Liu20092686} and Li et al. \cite{Li2016628}, where the consistency of a multiplicative IPCM is measured by computing Saaty's consistency index \cite{Saaty77} of one or two  associated  PCMs, index $I_{[\mathcal{G}]}$ is computed directly on the original IPCM (i.e. without considering associated PCMs), and it is suitable for each kind of IPCM (i.e. not only multiplicative IPCMs). For the cases of multiplicative, additive, fuzzy  IPCMs, $[\mathcal{G}]$-consistency index $I_{[\mathcal{G}]}$ assumes the following forms:
\begin{equation}\label{eq:multiplicativeConsistencyIndex}
 I_{[\mathcal{R}^+]} \left( \tilde{A} \right) = \left( \prod_{i<j<k} d_{[\mathcal{R}^+]}(\tilde{a}_{ijk},\tilde{a}_{ikj}) \right)^{\frac{6}{n(n-2)(n-1)}}, 
\end{equation}
\begin{equation}\label{eq:additiveConsistencyIndex}
 I_{[\mathcal{R}]} \left( \tilde{A} \right) = \frac{6}{n(n-2)(n-1)}  \sum_{i<j<k} d_{[\mathcal{R}]}(\tilde{a}_{ijk},\tilde{a}_{ikj}), 
\end{equation}
\begin{equation}\label{eq:fuzzyConsistencyIndex}
 I_{[\mathcal{I}]} \left( \tilde{A} \right) =\frac{ \left( \prod_{i<j<k} d_{[\mathcal{I}]}(\tilde{a}_{ijk},\tilde{a}_{ikj}) \right)^{\frac{6}{n(n-2)(n-1)}}}{ \left( \prod_{i<j<k} d_{[\mathcal{I}]}(\tilde{a}_{ijk},\tilde{a}_{ikj}) \right)^{\frac{6}{n(n-2)(n-1)}}
+  \left( \prod_{i<j<k} (1- d_{[\mathcal{I}]}(\tilde{a}_{ijk},\tilde{a}_{ikj})) \right)^{\frac{6}{n(n-2)(n-1)}}}.
\end{equation}
We remark that isomorphisms between Alo-groups allow us to compare consistency of IPCMs defined over different Alo-groups; e.g. in Example \ref{Example:Fuzzy_[G]consistencyIndex}, we will compare consistency of a multiplicative IPCM with consistency of a fuzzy IPCM.
%
%
%

\begin{example}\label{Example:Multiplicative_[G]consistencyIndex}
Let us consider the following multiplicative $[\mathcal{R^+}]$-reciprocal IPCM, which was used also by Arbel and Vargas \cite{ArbelVargas1993}, Haines \cite{Haines1998}, and Wang et al. \cite{WangEtAl2005} 
\[
\tilde{A}_{1}=
\begin{pmatrix}
[1,1] & [2,5] & [2,4] & [1,3] \\
[ \frac{1}{5},\frac{1}{2} ] & [1,1] & [1,3] & [1,2] \\
[ \frac{1}{4},\frac{1}{2} ] & [ \frac{1}{3},1 ] & [1,1] & [\frac{1}{2},1] \\
[ \frac{1}{3},1 ] & [ \frac{1}{2},1 ] & [1,2] & [1,1] 
\end{pmatrix}.
\]
By applying Definition \ref{def:[G]_consistencyIndex} and Proposition \ref{Prop:simplification_DG_aijk_aikj},   $ I_{ [\mathcal{R^+}] } \left( \tilde{A}_{1} \right) $ is given by the following geometric mean:
\begin{align*}
 I_{ [\mathcal{R^+}] } \left( \tilde{A}_{1} \right) = & \sqrt[4]{d_{[R^+]}(\tilde{a}_{123}, \tilde{a}_{132} )  \cdot    d_{[R^+]}(\tilde{a}_{124}, \tilde{a}_{142} ) \cdot    d_{[R^+]}(\tilde{a}_{134}, \tilde{a}_{143} ) \cdot    d_{[R^+]}(\tilde{a}_{234}, \tilde{a}_{243} )  } =\\
=& \sqrt[4]{\max \left\{\frac{a_{12}^- \cdot a_{23}^- \cdot a_{31}^- }{a_{13}^- \cdot a_{32}^- \cdot a_{21}^-}, \frac{a_{13}^- \cdot a_{32}^- \cdot a_{21}^-}{a_{12}^- \cdot a_{23}^- \cdot a_{31}^- }\right\}         \cdot   \max \left\{\frac{a_{12}^- \cdot a_{24}^- \cdot a_{41}^- }{a_{14}^- \cdot a_{42}^- \cdot a_{21}^-}, \frac  {a_{14}^- \cdot a_{42}^- \cdot a_{21}^-}  {a_{12}^- \cdot a_{24}^- \cdot a_{41}^- }     \right\}     } \cdot \\
\cdot&  \sqrt[4]{ \max \left\{\frac{a_{13}^- \cdot a_{34}^- \cdot a_{41}^- }{a_{14}^- \cdot a_{43}^- \cdot a_{31}^-}, \frac{a_{14}^- \cdot a_{43}^- \cdot a_{31}^-}{a_{13}^- \cdot a_{34}^- \cdot a_{41}^- }     \right\}   \cdot        \max \left\{\frac{a_{23}^- \cdot a_{34}^- \cdot a_{42}^- }{a_{24}^- \cdot a_{43}^- \cdot a_{32}^-}, \frac{a_{24}^- \cdot a_{43}^- \cdot a_{32}^-}{a_{23}^- \cdot a_{34}^- \cdot a_{42}^- }     \right\}               }=
 \\
 =&
\sqrt[4]{\frac{15}{4} \cdot \frac{20}{3} \cdot \frac{4}{3} \cdot \frac{4}{3} }
\approx 2.58199.
\end{align*}
\end{example}

\begin{example}\label{Example:Fuzzy_[G]consistencyIndex}
Let us consider the following  $[\mathcal{I}]$-reciprocal fuzzy IPCM, proposed by Wang and Li \cite{WangLi2012} 

\[
\tilde{A}_{2}=
\begin{pmatrix}
[0.50, 0.50] & [0.35,0.50] & [0.50,0.60] & [0.45,0.60] \\
[ 0.50, 0.65 ] & [0.50,0.50] & [0.55,0.70] & [0.50,0.70] \\
[ 0.40,0.50 ] & [ 0.30, 0.45 ] & [0.50,0.50] & [0.40,0.55] \\
[ 0.40,0.55 ] & [ 0.30,0.50 ] & [0.45,0.60] & [0.50,0.50]
\end{pmatrix}.
\]
Its consistency index can be computed by applying \eqref{eq:fuzzyConsistencyIndex}, or   by applying the isomorphism $h:  \mathbb{R}^{+}  \rightarrow ]0,1[$ in  \eqref{eq:isomorphism_psi_01}, that is:
 $$I_{ [\mathcal{I}]} \left( \tilde{A}_{2} \right) =  h(I_{\mathcal{[R^+]}}(h^{-1}(\tilde{A}_{2} )))\approx 0.503448.$$ Interestingly, although they are expressed on two different scales, values of consistency indices from different representations of preferences are, thanks to the isomorphisms, comparable. For instance,  let us consider the multiplicative IPCM in Example \ref{Example:Multiplicative_[G]consistencyIndex}; then, 
by using the isomorphism $h$ in  \eqref{eq:isomorphism_psi_01}, we have that
\[
h \left( I_{ [\mathcal{R^+}]} \left( \tilde{A}_{1} \right) \right) \approx 0.720826,
\]
which entails that $\tilde{A}_{1}$ is more inconsistent than $\tilde{A}_{2}$.
\end{example}

\section{$[\mathcal{G}]$-Indeterminacy index of $[\mathcal{G}]$-reciprocal IPCMs}
\label{sec:indeterminacy}

If we follow the definition of $[\mathcal{G}]$-consistency (Definition \ref{Def:consistent_IPCM}), we could encounter cases where IPCMs with extremely wide intervals are considered $[\mathcal{G}]$-consistent. Li et al.\cite{Li2016628} and Zhang \cite{Zhang2016} 
reckoned that a multiplicative IPCM with all non-diagonal entries equal to $\tilde{a}_{ij}=[1/9,9]$ would be considered consistent. This case reflects a high ambiguity from the decision maker's side and is classified as fully consistent. Thus, in all these cases, the consistency of a IPCM, as formulated in Definition \ref{Def:consistent_IPCM} loses its capacity of yielding information on the real ability of a decision maker to be rational. To mitigate this problem, 
Li et al. \cite{Li2016628}
 suggested the use of an index of indeterminacy. Inconsistency of interval-valued preferences and width of the intervals can then be used in concert to better asses the discriminative capacity of a decision maker.
We shall here propose a general definition of indeterminacy index and show that the proposals by Li et al.~\cite{Li2016628} and Zhang \cite{Zhang2016}
 fits within in.

\begin{definition}\label{def:indeterminancy_index_3}
Let $\tilde{A}$ be a $[\mathcal{G}]$-reciprocal IPCM. The indeterminacy value of the entry $\tilde{a}_{ij}$ is  
\begin{equation}
\label{eq:dG}
\delta(\tilde{a}_{ij})=d_\mathcal{G} (a_{ij}^{-}, a_{ij}^{+}).
\end{equation}
\end{definition}

\begin{corollary}\label{coro:deltaIndetrminacy}
Let $\tilde{A}$ be a $[\mathcal{G}]$-reciprocal IPCM; then the following equality holds:
\begin{equation}
\delta(\tilde{a}_{ij})=a_{ij}^{+} \div a_{ij}^{-}.
\end{equation}
\end{corollary}

%
%

\begin{definition}\label{def:indeterminancy_index}
Let $\tilde{A}$ be a $[\mathcal{G}]$-reciprocal IPCM.  The $[\mathcal{G}]$-indeterminacy index is  
\begin{equation*}
\Delta_{[\mathcal{G}]} \left( \tilde{A} \right) =\left( \bigodot_{i \neq j} \delta (\tilde{a}_{ij})\right)^{(\frac{1}{n(n-1)})}.
\end{equation*}
\end{definition}

We stress that,  in  Definition \ref{def:indeterminancy_index}, $n(n-1) \in \N$ and  the $[\mathcal{G}]$-indeterminacy index is a $\mathcal{G}$-mean (see Definition \ref{def:G_mean}) of   $n(n-1)$ indeterminacy values.  Moreover, 
let  $\phi$ be an isomorphism between $\mathcal{G}= (G,
\odot, \leq)$ and $\mathcal{H}= (H, *, \leq)$, $\tilde{A}'=\phi(\tilde{A})=(\phi(\tilde{a}_{ij}))$; then, by  \eqref{eq:mean_isomorphism}, we have that:
\begin{equation}
\Delta_{\mathcal{[H]}}(\tilde{A}')=\phi(\Delta_{\mathcal{[G]}}(\tilde{A})).
\end{equation}

Furthermore, by  Corollary \ref{cor:alternativa3}, we have that $a_{ij}^{+} \div a_{ij}^{-} = a_{ji}^{+} \div a_{ji}^{-}$; thus, 
it is sufficient to consider the comparisons in the upper triangle of $\tilde{A}$ and it leads to a simplification of the previous formula into:
\begin{equation}\label{eq:IndeterminancyIndex}
\Delta_{[\mathcal{G}]} \left( \tilde{A} \right)=\left( \bigodot_{i < j} \left( a_{ij}^{+} \div  a_{ij}^{-} \right)^{(2)} \right)^{(\frac{1}{n(n-1)})}.
\end{equation}
%
%
%
\begin{proposition}\label{prop:unique_element_for_[G]indeterminancy}
Let $\tilde{A}$ be a $[\mathcal{G}]$-reciprocal IPCM; then: 
\begin{equation*}
\label{ eq:IG_>_e2}
\Delta_{[\mathcal{G}]} \left( \tilde{A} \right) \geq e, \quad  \Delta_{[\mathcal{G}]} \left( \tilde{A} \right) =e   \Leftrightarrow  \tilde{A} \in [G]_p.
\end{equation*}
\end{proposition}
%
%
For multiplicative, additive and fuzzy IPCMs, the indeterminacy index can be written as follows, respectively,
\begin{equation}\label{eq:MultiplicativeIndeterminancyIndex}
\Delta_{[\mathcal{R}^+]} \left( \tilde{A} \right)=\left( \prod_{i < j} \left( \frac{a_{ij}^{+}}{ a_{ij}^{-} }  \right) \right)^{\frac{2}{n(n-1)}}, 
\end{equation}
\begin{equation}\label{eq:AdditiveIndeterminancyIndex}
\Delta_{[\mathcal{R}]} \left( \tilde{A} \right)=  \frac{2}{n(n-1)} \sum_{i < j} \left(a_{ij}^{+} - a_{ij}^{-}     \right) , 
\end{equation}
\begin{equation}\label{eq:FuzzyIndeterminancyIndex}
\Delta_{[\mathcal{I}]} \left( \tilde{A} \right)=\frac{\left( \prod_{i < j}  \frac{a_{ij}^{+}(1-a_{ij}^{-})}{a_{ij}^{+}(1-a_{ij}^{-})+(1-a_{ij}^{+})a_{ij}^{-}} \right)^{\frac{2}{n(n-1)}} }    {\left( \prod_{i < j}  \frac{a_{ij}^{+}(1-a_{ij}^{-})}{a_{ij}^{+}(1-a_{ij}^{-})+(1-a_{ij}^{+})a_{ij}^{-}}  \right)^{\frac{2}{n(n-1)}}
+ \left( \prod_{i < j} \left( 1-\frac{a_{ij}^{+}(1-a_{ij}^{-})}{a_{ij}^{+}(1-a_{ij}^{-})+(1-a_{ij}^{+})a_{ij}^{-}} \right) \right)^{\frac{2}{n(n-1)}}}.
\end{equation}
It is worth noting that:
\begin{itemize}
\item The indeterminacy index (\ref{eq:MultiplicativeIndeterminancyIndex}) is equal to the indeterminacy index proposed by Li et al.~\cite{Li2016628}; 
  \item Function \eqref{eq:AdditiveIndeterminancyIndex} represents a  way to measure the indeterminacy of an additive IPCM and to best of our knowledge there has not been similar proposals for the additive approach in the literature;
 \item Equation (\ref{eq:FuzzyIndeterminancyIndex}) is different from the indeterminacy index proposed by Wang and Chen \cite{WANG201459}  because (\ref{eq:FuzzyIndeterminancyIndex}) takes in account the fuzzy mean, instead of geometric mean, and inverse of fuzzy group operation $\otimes$, instead of classical division (the set $]0,1[$ is not closed under the classical division of $\R$). 
\end{itemize} 
Finally, the $[\mathcal{G}]$-indeterminacy index is suitable for each kind of IPCM (i.e. not only multiplicative and fuzzy IPCMs) and isomorphisms between Alo-groups allow us to compare indeterminacy of IPCMs defined over different Alo-groups. In the following example we compare indeterminacy of a multiplicative IPCM with indeterminacy of a fuzzy IPCM.

\begin{example} \label{Example:indeterminacyIndex}
Let us consider the multiplicative IPCM $\tilde{A}_{1}$ in Example \ref{Example:Multiplicative_[G]consistencyIndex} and the fuzzy IPCM in Example \ref{Example:Fuzzy_[G]consistencyIndex}.
Firstly, we compute  the $[\mathcal{R}^{+}]$-indeterminacy index of $\tilde{A}_{1}$ as follows:

\begin{align*}
\Delta_{[\mathcal{R}^{+}]}(\tilde{A}_{1}) = \left( \prod_{i<j} \left( \frac{a_{ij}^{+}}{a_{ij}^{-}}\right)^2 \right)^{\frac{1}{4(4-1)}}
              & = \left(      \frac{5}{2} \cdot \frac{4}{2} \cdot \frac{3}{1} \cdot \frac{3}{1} \cdot \frac{2}{1} \cdot \frac{1}{\frac{1}{2}}     \right) ^{1/6} = 180^{1/6} \approx 2.376.
\end{align*}

Secondly, we   compute the $[\mathcal{I}]$-indeterminacy index of $\tilde{A}_{2}$, by applying the isomorphism $h$ in  \eqref{eq:isomorphism_psi_01}, that is: 
$$\Delta_{[\mathcal{I]}} \left( \tilde{A}_{2} \right) =  h(\Delta_{\mathcal{[R^+]}}(h^{-1}(\tilde{A}_{2} )))\approx  0.6506.$$

In their present forms the two values are incomparable, but they can be made comparable by applying   the isomorphism $h$ to $\Delta_{[\mathcal{R}^{+}]}(\tilde{A}_{1})$ and obtain 
\[
h \left( \Delta_{[\mathcal{R}^{+}]}\left(\tilde{A}_{1}\right) \right) \approx 0.7038
\]
which, being larger than $0.6506$, indicates that globally the preferences contained in the multiplicative IPCM $\tilde{A}_{1}$ are more indeterminate than those contained in the fuzzy IPCM $\tilde{A}_{2}$, in addition to being more inconsistent (see Example \ref{Example:Fuzzy_[G]consistencyIndex})
\end{example}

All in all, it has been stipulated that we can associate a consistency and an indeterminacy value to each IPCM. In line with the approach by Li et al. \cite{Li2016628}, we also propose to use both values to determine whether or not a matrix needs revision.
To this end, we observe that the conjoint use of both indices lends itself to some graphical interpretations. 
\begin{itemize}

\item For each matrix $\tilde{A}$ we have two values, $I_{[\mathcal{G}]}(\tilde{A})$ and $\Delta_{[\mathcal{G}]}(\tilde{A})$. As shown in Figure \ref{fig:1}, these two values partition the graph into four subsets. IPCMs with values in $Q_1$ have greater indeterminacy and inconsistency than $\tilde{A}$. Therefore it seems reasonable to consider them \emph{more} inaccurate/irrational. With a similar reasoning one could classify IPCMs with values in $Q_3$ as \emph{less} inaccurate/irrational. IPCMs in $Q_{2}$ and $Q_{4}$ are not comparable since they have one value which is greater, but the other one which is smaller.

\item The second interpretation, also proposed by Li et al. \cite{Li2016628} is that of fixing thresholds for both indices and accept only IPCMs whose values are smaller or equal than the thresholds. Namely, if ${t}_{I}$ and ${t}_{\Delta}$ were the thresholds, then we should accept only the IPCMs in the grey area in Figure \ref{fig:2}.

\end{itemize}

\begin{figure}
    \centering
    \begin{subfigure}[b]{0.45\textwidth}
		 \centering
        \includegraphics[scale=1]{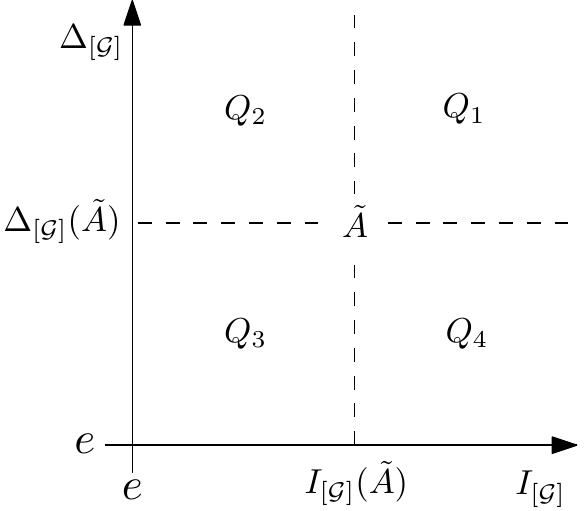}
        \caption{First interpretation}
        \label{fig:1}
    \end{subfigure}
    ~ 
    \begin{subfigure}[b]{0.45\textwidth}
		 \centering
        \includegraphics[scale=1]{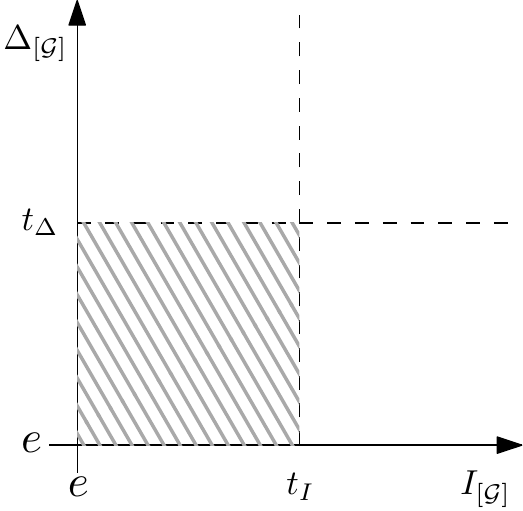}
        \caption{Second interpretation}
        \label{fig:2}
    \end{subfigure}
    \caption{ $I_{[\mathcal{G}]}$-$\Delta_{[\mathcal{G}]}$ axes: $I_{[\mathcal{G}]}$ is the $[\mathcal{G}]$-consistency index;  $\Delta_{[\mathcal{G}]}$ is the $[\mathcal{G}]$-indeterminacy index.   Lying on the $I_{[\mathcal{G}]}$ axis, there are all PCMs. Lying on the $\Delta$-axis, there are all $[\mathcal{G}]$-consistent IPCMs. On the origin of the axes $(e,e)$, there  are all $\mathcal{G}$-consistent PCMs.}   \label{fig:interpretations}
\end{figure}

\begin{example}
Let us consider the multiplicative and fuzzy IPCMs $\tilde{A}_{1}$ and $\tilde{A}_{2}$ used in Examples \ref{Example:Multiplicative_[G]consistencyIndex}, \ref{Example:Fuzzy_[G]consistencyIndex}, and \ref{Example:indeterminacyIndex}. For $\tilde{A}_{2}$, we had:
\begin{align*}
I_{[\mathcal{I}]}(\tilde{A}_{2}) &= 0.503448;\\
\Delta_{[\mathcal{I}]}(\tilde{A}_{2}) &= 0.6506.
\end{align*}
For $\tilde{A}_1$, by using the proper isomorphism $h$ mentioned in (\ref{eq:isomorphism_psi_01}), we obtained:
\begin{align*}
I_{[\mathcal{R}^{+}]}(\tilde{A}_{1}) = 2.58199 & \Rightarrow h(2.582) =
 0.720826 = I_{[\mathcal{I}]}(h(\tilde{A}_{1})); \\
\Delta_{[\mathcal{R}^{+}]}(\tilde{A}_{1}) = 2.376 & \Rightarrow h(2.376) =
 0.7038 = \Delta_{[\mathcal{I}]}(h(\tilde{A}_{1})).
\end{align*}
For sake of completeness, we also consider the following additive IPCM
\[
\tilde{A}_{3}=
\begin{pmatrix}
[0,0] & [1,3] & [2,4] & [6,8] \\
[ -3,-1 ] & [0,0] & [1,3] & [4,5] \\
[ -4,-2 ] & [ -3, -1 ] & [0,0] & [2,3] \\
[ -8,-6 ] & [ -5,-4 ] & [-3,-2] & [0,0]
\end{pmatrix};
\]
for which, by using the isomorphism $g$ presented in (\ref{eq:isomorphism_psi_01}), we can derive:
\begin{align*}
I_{[\mathcal{R}]}(\tilde{A}_{3}) = 3/2 & \Rightarrow g(3/2) =
 0.8175 = I_{[\mathcal{I}]}(g(\tilde{A}_{3})); \\
\Delta_{[\mathcal{R}]}(\tilde{A}_{3}) = 5/3 & \Rightarrow g(5/3) =
 0.841131 = \Delta_{[\mathcal{I}]}(g(\tilde{A}_{3})).
\end{align*}
At this point, thanks to the isomorphisms, we can give a common graphical interpretation of the levels of inconsistency and indeterminacy of all IPCMs, whether they be multiplicative, additive or fuzzy. In this example, we can position the preferences expressed in $\tilde{A}_1,\tilde{A}_2,\tilde{A}_3$ on  $[0.5,1[ \times [0.5,1[$. Figure \ref{fig:3} represents the \lq\lq dominance\rq\rq of the preferences expressed in $\tilde{A}_2$ over those expressed in $\tilde{A}_1$ since $\tilde{A}_2$ is both less inconsistent and less indeterminate than $h(\tilde{A}_1)$. The same can be said of the preferences of $\tilde{A}_2$ when compared to those in $\tilde{A}_3$.\\
If we establish thresholds $t_{I}=0.7$ and $t_{\Delta}=0.7$ and stipulate that an acceptable IPCM ought to satisfy both of them, then Figure \ref{fig:4} shows that only the preferences contained in $\tilde{A}_2$ should be considered acceptable.
\begin{figure}
    \centering
    \begin{subfigure}[t]{0.45\textwidth}
		 \centering
        \includegraphics[scale=1]{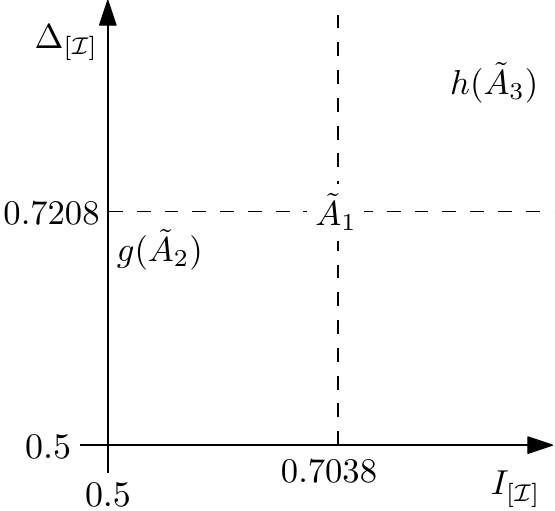}
        \caption{Preferences expressed in $\tilde{A}_2$ are less inconsistent and less indeterminate than those in $\tilde{A}_1$.
Preferences expressed in $\tilde{A}_1$ are less inconsistent and less indeterminate than those in $\tilde{A}_3$.        
}
        \label{fig:3}
    \end{subfigure}
    ~ 
    \begin{subfigure}[t]{0.45\textwidth}
		 \centering
        \includegraphics[scale=1]{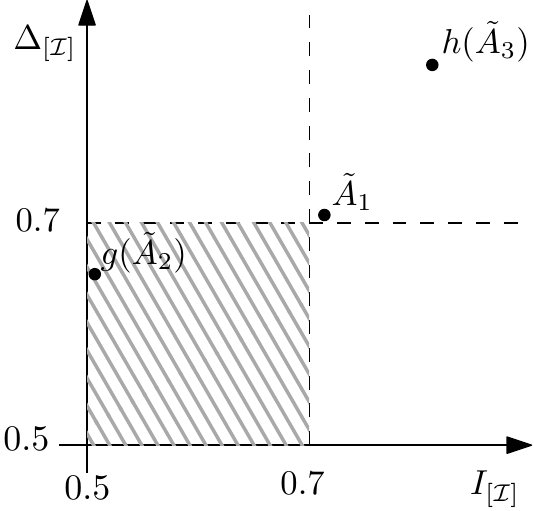}
        \caption{With thresholds $t_{I}=0.7$ and $t_{\Delta}=0.7$ only $\tilde{A}_2$ is considered acceptable.}
        \label{fig:4}
    \end{subfigure}
    \caption{ $I_{[\mathcal{I}]}$-$\Delta_{[\mathcal{I}]}$ axes: $I_{[\mathcal{I}]}$ is the $[\mathcal{I}]$-consistency index;  $\Delta_{[\mathcal{I}]}$ is the $[\mathcal{I}]$-indeterminacy index. Graphical analysis of inconsistency and indeterminacy of matrices $\tilde{A}_1$,  $\tilde{A}_2$ and $\tilde{A}_3$.}   \label{fig:interpretations_ex}
\end{figure}
\end{example}

\section{Conclusions and future work}
\label{sec:conclusions}
In the paper, after generalizing interval arithmetic to a suitable algebraic structure,  we provide a general unified framework for dealing with IPCMs; in particular, reciprocal IPCMs, whose entries are intervals on real continuous Abelian linearly ordered groups, allow us to unify several approaches proposed in the literature, such as multiplicative, additive and fuzzy IPCMs.\\ In this context, firstly, we generalize some consistency conditions proposed in the literature and we establish  inclusion relations between them. 
Then, we provide a consistency index, based on a concept of distance between intervals, in order to asses how much an IPCM is far from consistency; this consistency index generalizes a consistency index proposed in \cite{CavalloDapuzzo}, \cite{CavalloDapuzzoSquillanteIJIS3} for PCMs. 
We also consider an indeterminacy index in order to assess ambiguity of a decision maker in expressing his/her preferences; consistency index and indeterminacy index are used in concert to assess the discriminative capacity of a decision maker 
and  isomorphisms between Alo-groups allow us to compare consistency and indeterminacy of each kind of IPCM and to represent them on a unique Cartesian coordinate system.
\\
Our future work will be  directed to investigate the possibility to  extend to this kind of IPCMs further notions and results  obtained in the context of PCMs defined over Abelian linearly ordered groups, such as   weighting vector \cite{CavalloDapuzzoSC}, transitivity condition  \cite{CavalloDapuzzoFSS} and weak consistency \cite{CavalloDapuzzoFSS2}.

Finally, we observe that besides the approach based on intervals, a seemingly different approach grounded on Atanassov's concept of intuitionistic fuzzy sets has gained prominence to extend the concepts of fuzzy PCMs \cite{Xu2007} and multiplicative PCMs \cite{Xu2013,XiaXuLiao2013}. Having realized this, thanks to an isomorphism between interval-valued fuzzy sets and intuitionistic fuzzy sets (see \cite{DeschrijverKerre2003}, Th. 2.3), the results and the methods developed in this paper can be straightforwardly extended to the case of intuitionistic PCMs.

\section*{Acknowledgement}
The research  is supported by \lq\lq Programma di scambi internazionali con Universit\`{a} e Istituti di Ricerca Stranieri per la Mobilit\`{a}  di breve durata di Docenti, Studiosi e Ricercatori\rq\rq, University of Naples \lq\lq Federico II\rq\rq. The research of Matteo Brunelli was financed by the Academy of Finland (decision no.~277135). We thank Prof. Raimo P. H\"{a}m\"{a}l\"{a}inen for having invited Bice Cavallo as visiting researcher at Aalto University and for having built the opportunity of research collaboration between the two Universities.









\section*{Appendix}
\begin{proof}[Proof of Theorem \ref{Theorem_productIntervals}]
We first consider the first equality. Since  $\odot$ is continuous on  $ [a^{-},a^{+}] \times [b^{-},b^{+}]$, there exist minimum and maximum in $ [a^{-},a^{+}] \times [b^{-},b^{+}]$, and  $\odot$ assumes all values between  minimum and maximum. Finally, since  $\odot$ is monotone increasing in
both arguments, the minimum is equal to $a^{-} \odot  b^{-}$ and maximum is equal to $a^{+} \odot  b^{+}$; thus, the assertion is achieved.
The second equality follows from the previous one and from \eqref{eq:inverseOperationIntervals}
\end{proof}
%
%
\begin{proof}[Proof of Proposition \ref{Prop:noInverse}]
\begin{enumerate}
\item By Theorem \ref{Theorem_productIntervals}, 
 $\tilde{a} \odot_ {[G]} [e,e]= [e,e]  \odot_ {[G]} \tilde{a}= \tilde{a} , \forall \tilde{a} \in [G]$.
 \item  $\Rightarrow$) Let us suppose that $\exists \tilde{b} \in [G]$ such that $\tilde{a} \odot_ {[{G}]}  \tilde{b} = \tilde{b} \odot_ {[{G}]}  \tilde{a}  =  [e,e]$.
Then, we have,
\[
[a^{-},a^{+}] \odot_ {[{G}]}  [b^{-},b^{+}] = [a^{-} \odot b^{-} , a^{+} \odot b^{+} ] = [e,e],
\]
from which
\begin{equation}
\label{eq:cases}
\begin{cases}
a^{-} \odot b^{-} = e \\
a^{+} \odot b^{+} = e
\end{cases} \Rightarrow 
\begin{cases}
a^{-} = (b^{-})^{(-1)} \\
a^{+} = (b^{+})^{(-1)}
\end{cases}
\end{equation}
By $a^{-} \leq a^{+}$, we have that $(b^{-})^{(-1)} \leq (b^{+})^{(-1)}$ and, by first equivalence in \eqref{composition-1},  $ b^{-} \geq b^{+}$, which, together with $b^{-} \leq b^{+}$ implies $b^{-} = b^{+}$. This, with (\ref{eq:cases}), implies that $a^{-} = a^{+}$ and $\tilde{a} \in [G]_{p}$.\\
$\Leftarrow$) Let us consider $\tilde{a}=[a,a]\in [G]_{p}$. Then, $[a,a] \odot_ {[{G}]}  [a^{(-1)},a^{(-1)}]=[e,e]$  and the assertion is achieved.
\end{enumerate}
\end{proof}
\begin{proof}[Proof of Theorem \ref{Theorem:isomorphism_intevals_Alo_group_G}]
 By Theorem \ref{Theorem_productIntervals},  $\tilde{a} \odot_ {[G]}  \tilde{b} \in [G]$ for each  $\tilde{a} , \tilde{b} \in [G]$. Associativity and commutativity of $\odot_ {[G]}$ follow by associativity and commutativity of  $\odot$. Finally, by Proposition \ref{Prop:noInverse}, 
%
$[\mathcal{G}]_p= ([G]_p, \odot_ {[G]})$ is an Abelian group. Moreover, by \eqref{eq:partial_weak_order_intervals} and \eqref{ordercons}, we have:
$$[a,a]\leq_{[G]} [b,b] \Leftrightarrow a \leq b \Leftrightarrow a\odot c \leq b\odot c \Leftrightarrow [a,a]    \odot_ {[{G}]}    [c,c] \leq_{[G]} [b,b] \odot_{[{G}]} [c,c];$$
thus, $[\mathcal{G}]_p= ([G]_p, \odot_ {[{G}]}, \leq_{[G]} )$ is an Alo-group.
\\The bijection
$$i: a \in G \mapsto [a,a] \in [G]_p $$
is a group isomorphism because
$$i(a \odot b)=[a\odot b, a\odot b] =[a,a] \odot_ {[{G}]} [b,b] =i(a) \odot_ {[{G}]} i(b),$$
and a lattice isomorphism because
$$a \leq b \Leftrightarrow  [a,a]  \leq_{[G]} [b,b] \Leftrightarrow i(a) \leq_{[G]} i(b);$$
thus, the assertion is achieved.
\end{proof}
\begin{proof}[Proof of Proposition \ref{Prop:propertiesNorm}]
Properties $1, 2, 3,4$ follow immediately from Definition
\ref{norm_intervals} and Proposition \ref{triangle}. By  Theorem \ref{Theorem_productIntervals}, Definition
\ref{norm_intervals} and  property 5 of Proposition \ref{triangle}, we have:
\begin{align*}
||\tilde{a}\odot_ {[\mathcal{G}]} \tilde{b}||_{[\mathcal{G}]}& = ||[a^- \odot b^-, a^+ \odot b^+]||_{[\mathcal{G}]} 
                                            = \max\{||a^- \odot b^-|| _{\mathcal{G}}, ||a^+ \odot b^+|| _{\mathcal{G}} \}\  \leq\\
                                            & \leq \max\{||a^-|| _{\mathcal{G}} \odot ||b^-|| _{\mathcal{G}}  , ||a^+|| _{\mathcal{G}} \odot  ||b^+|| _{\mathcal{G}} \} \leq  \max \{   ||a^-|| _{\mathcal{G}} , ||a^+|| _{\mathcal{G}}    \} \odot  \max \{   ||b^-|| _{\mathcal{G}} , ||b^+|| _{\mathcal{G}}    \}=\\
																						& =  ||\tilde{a}||_{[\mathcal{G}]}\odot ||\tilde{b}||_{[\mathcal{G}]};
\end{align*}
thus, item 5 is achieved.
\end{proof}
\begin{proof}[Proof of Proposition \ref{prop:dis_intervals}]
By Proposition \ref{Prop:propertiesNorm},  properties 1--3 in Definition \ref{def:dis_intervals} are satisfied.\\
Let us consider $\tilde{a}=[a^-,a^+],  \tilde{b}=[b^-,b^+],   \tilde{c}=[c^-,c^+] \in [\mathcal{G}]$; then, by Definition \ref{norm_intervals}, Proposition \ref{dcirc} and property 4 in Definition \ref{dis},  we have:
\begin{align*} \label{eq_CG_distance}
d_{[\mathcal{G}]}(\tilde{a},\tilde{b})
&= \max\{ d_{\mathcal{G}}(a^{-}, b^{-}), d_{\mathcal{G}}(a^{+}, b^{+}) \} \\
&\leq \max \{ d_{\mathcal{G}}(a^{-}, c^{-})  \odot d_{\mathcal{G}}(c^{-}, b^{-}) ,  d_{\mathcal{G}}(a^{+}, c^{+})  \odot d_{\mathcal{G}}(c^{+}, b^{+}) \}  \\
&\leq \max \{d_{\mathcal{G}}(a^{-}, c^{-}) , d_{\mathcal{G}}(a^{+}, c^{+})  \} \odot \max \{d_{\mathcal{G}}(c^{-}, b^{-}) , d_{\mathcal{G}}(c^{+}, b^{+})  \} \\
&= d_{[\mathcal{G}]}(\tilde{a},\tilde{c}) \odot d_{[\mathcal{G}]}(\tilde{c},\tilde{b});
\end{align*}
thus, the assertion is achieved.
\end{proof}
\begin{proof}[Proof of Corollary \ref{cor:alternativa3}]
By  Definition \ref{cor:alternativa3} and \eqref{eq:reciprocalInterval}.
\end{proof}
\begin{proof}[Proof of Proposition \ref{Prop:IPCMdegenratesPCM}]
$\Rightarrow$) By  
\eqref{eq:cond} and Proposition  \ref{Prop:noInverse}, we have that  $\tilde{a}_{ij} \in [G]_p$; thus, by setting  $a_{ij}=a_{ij}^{-}=a_{ij}^{+} $ and by applying Corollary \ref{cor:alternativa3}, we have $a_{ij} \odot  a_{ji}  =e$ (i.e.  $A=( a_{ij})$ is a $\mathcal{G}$-reciprocal PCM).\\
$\Leftarrow$)
The assertion follows by:
$$\tilde{a}_{ij} \odot_{[\mathcal{G}]} \tilde{a}_{ji}=[a_{ij},  a_{ij}] \odot_{[\mathcal{G}]} [a_{ji},  a_{ji}]=[e,e]$$
and
$$\tilde{a}_{ji} =[a_{ji}, a_{ji}]=[a_{ij}^{(-1)}, a_{ij}^{(-1)}] =\tilde{a}_{ij}^{(-1)}.$$
\end{proof}
\begin{proof}[Proof of Proposition \ref{Prop:reciprocalIPCM_permutations}]
Let $\tilde{A}=(\tilde{a}_{ij})$ be a  $[\mathcal{G}]$-reciprocal IPCM; then, by  Definition \ref{def:alternativa3},  for each permutation  $\sigma$,  the following equalities hold true:
$$ \tilde{a}_{\sigma(j)\sigma(i)}=\tilde{a}_{\sigma(i)\sigma(j)}^{(-1)}~ \forall \; i,j\in \{1, \ldots, n\}$$ and, as a consequence,  $\tilde{A}^{\sigma}$ is  $[\mathcal{G}]$-reciprocal. \\
The vice versa is straightforward.
\end{proof}
\begin{proof}[Proof  of Proposition \ref{prop:LiuEquivalent}]
$1. \Leftrightarrow 2.$By  Proposition \ref{prop:consistency_i<j<k}.\\
$2. \Rightarrow 3. $
For each $i<j<k$, we have:
\begin{align*}
\tilde{a}_{ik} =[a^-_{ik},a^+_{ik} ]= [l_{ik},r_{ik} ]=[l_{ij} \odot l_{jk}, r_{ij} \odot r_{jk} ]=[l_{ij} , r_{ij} ] \odot_ {[\mathcal{G}]}  [l_{jk} , r_{jk} ] =[a^-_{ij},a^+_{ij} ]\odot_ {[\mathcal{G}]}  [a^-_{jk},a^+_{jk} ]= \tilde{a}_{ij} \odot_ {[\mathcal{G}]} \tilde{a}_{jk}.
\end{align*}
$3. \Rightarrow2.$ By assumption, for each $i<j<k$, we have:
$$[a^-_{ik},a^+_{ik} ]= [a^-_{ij},a^+_{ij} ]\odot_ {[\mathcal{G}]}  [a^-_{jk},a^+_{jk} ]=[a^-_{ij} \odot a^-_{jk}  ,a^+_{ij} \odot a^+_{jk} ];$$
thus:
$$
\begin{cases}
l_{ik}=  a^-_{ik}= a^-_{ij} \odot a^-_{jk}=   l_{ij} \odot l_{jk} \\
r_{ik}= a^+_{ik} =  a^+_{ij} \odot a^+_{jk} =                r_{ij} \odot r_{jk}
\end{cases}
\;\;
\forall i<j<k.
$$
\end{proof}
\begin{proof}[Proof of Theorem \ref{ApproximateConsistencyDegeneration}]
If $\tilde{A}=([a_{ij}^-, a_{ij}^+])$ degenerates in a $\mathcal{G}$-consistent PCM  $A=(a_{ij})$ over $(G, \odot, \leq)$ with  $a_{ij}^-= a_{ij}^+=a_{ij}$, for each $i,j =1, \ldots, n$, then by Proposition \ref{prop:unique_element_for_consistency}, $A^\sigma=(a_{\sigma(i)\sigma(j)})$ is a $\mathcal{G}$-consistent PCM over $(G, \odot, \leq)$ for each permutation $\sigma$. Thus,   $L^\sigma=(l_{ij}^\sigma)=R^\sigma=(r_{ij}^\sigma)$   is a $\mathcal{G}$-consistent PCM over $(G, \odot, \leq)$ for each permutation $\sigma$.\\
Viceversa, let us assume  $L^\sigma=(l_{ij}^\sigma)$ and $R^\sigma=(r_{ij}^\sigma)$   be $\mathcal{G}$-consistent PCMs over $(G, \odot, \leq)$ for all permutation $\sigma$; thus, for  each permutation $\sigma$, the following equalities hold: 
\begin{equation} \label{eq:L_R_sigma_consistent}
l_{ik}^\sigma=l_{ij}^\sigma \odot l_{jk}^\sigma \quad \quad r_{ik}^\sigma=r_{ij}^\sigma \odot r_{jk}^\sigma.
\end{equation}
 Without loss of generality,  we can assume $i_1<k_1<j_1$; thus, for a permutation $\sigma_1$, by definition of $l_{ij}^\sigma$ in 
(\ref{eq:L_R_sigma}) and first equality in (\ref{eq:L_R_sigma_consistent}), we have:
$$ a_{\sigma_1(i_1)\sigma_1(k_1)}^-=a_{\sigma_1(i_1)\sigma_1(j_1)}^- \odot a_{\sigma_1(j_1)\sigma_1(k_1)}^+ .$$
Let us consider a permutation $\sigma_2$ and integers $i_2, j_2, k_2$, with $j_2<i_2<k_2$, 
such that:
$$l_{i_1j_1}^{\sigma_1}=l_{i_2j_2}^{\sigma_2}, \;   l_{j_1k_1}^{\sigma_1}=l_{j_2k_2}^{\sigma_2}, \; l_{i_1k_1}^{\sigma_1}=l_{i_2k_2}^{\sigma_2},$$
$$r_{i_1j_1}^{\sigma_1}=r_{i_2j_2}^{\sigma_2}, \;   r_{j_1k_1}^{\sigma_1}=r_{j_2k_2}^{\sigma_2}, \; r_{i_1k_1}^{\sigma_1}=r_{i_2k_2}^{\sigma_2}.$$
Thus, by definition of $r_{ij}^\sigma$ in
(\ref{eq:L_R_sigma}) and second equality in (\ref{eq:L_R_sigma_consistent}), we have:
$$a_{\sigma_1(i_1)\sigma_1(k_1)}^+ = r_{i_1k_1}^{\sigma_1}=r_{i_2k_2}^{\sigma_2} 
= r_{i_2j_2}^{\sigma_2} \odot r_{j_2k_2}^{\sigma_2}= a_{\sigma_2(i_2)\sigma_2(j_2)}^- \odot a_{\sigma_2(j_2)\sigma_2(k_2)}^+ = a_{\sigma_1(i_1)\sigma_1(j_1)}^- \odot a_{\sigma_1(j_1)\sigma_1(k_1)}^+ . $$
Thus, $a_{\sigma(i)\sigma(k)}^-=a_{\sigma(i)\sigma(k)}^+$ ($i,k \in \{1, \ldots, n\}$) for each permutation $\sigma$, and, as a consequence, the assertion is achieved.
\end{proof}
\begin{proof}[Proof of Proposition \ref{Prop:invariance_odotG_consistency}]
Checking the $[\mathcal{G}]$-consistency of $\tilde{A}$ requires checking that condition (\ref{eq:consistent_IPCM}) holds for the set of triples in the set $S=\{ (i,j,k) \, | \, i,j,k \in \{ 1,\ldots,n \} \}$. Similarly, $[\mathcal{G}]$-consistency of $\tilde{A}^{\sigma}$ requires that condition (\ref{eq:consistent_IPCM}) hold for all the triples in $S^{\sigma}=\{ (\sigma(i),\sigma(j),\sigma(k)) \, | \, i,j,k \in \{ 1,\ldots,n \} \}$. Since by definition $\sigma:\{ 1,\ldots,n\} \rightarrow \{ 1,\ldots,n \}$ is a bijection, we know that $S = S^{\sigma}$, and hence the proposition is true.
\end{proof}
\begin{proof}[Proof of Theorem \ref{Theorem_odotG_consistency_equivalences}]
$1 \Rightarrow 2$
Let us assume \eqref{eq:Equivalence_odotGConsistency} be true. Then, by applying   $[\mathcal{G}]$-reciprocity $a_{ij}^- \odot a_{ji}^+=a_{ij}^+ \odot a_{ji}^-=e$ (Corollary \ref{cor:alternativa3}), we have:
\begin{align*}
(a_{ik}^- \odot \underbrace{a_{ik}^+) \odot (a_{ki}^{-}}_e \odot a_{kj}^{-} \odot a_{ji}^{-})=(\underbrace{a_{ik}^- \odot  a_{kj}^{-} \odot a_{ji}^{-}}_{  {a}_{ij}^{-} \odot {a}_{jk}^{-} \odot {a}_{ki}^{-} })  \odot \underbrace{a_{kj}^- \odot a_{jk}^+}_e \odot  \underbrace{a_{ji}^- \odot a_{ij}^+}_e =(a_{ij}^- \odot a_{ij}^+ \odot a_{jk}^- \odot a_{jk}^+)  \odot  (a_{ki}^{-} \odot a_{kj}^{-} \odot a_{ji}^{-} );
\end{align*}
thus, by cancellative law, the assertion is achieved.\\
$2 \Rightarrow 1$
By applying   $[\mathcal{G}]$-reciprocity $a_{ij}^- \odot a_{ji}^+=a_{ij}^+ \odot a_{ji}^-=e$ (Corollary \ref{cor:alternativa3}), 
for each $ i,j,k \in \{1, \ldots, n\}$,  we have:
\begin{align*}
  {a}_{ij}^{-} \odot {a}_{jk}^{-} \odot {a}_{ki}^{-} =&  {a}_{ij}^{-} \odot {a}_{jk}^{-} \odot {a}_{ki}^{-}  \odot \underbrace{{a}_{ji}^{-} \odot {a}_{ij}^{+} }_e \odot  \underbrace{{a}_{kj}^{-} \odot {a}_{jk}^{+} }_e=
   \underbrace{  {a}_{ij}^{-} \odot {a}_{ij}^{+} \odot {a}_{jk}^{-} \odot {a}_{jk}^{+} }_{a_{ik}^- \odot a_{ik}^+} \odot
  {a}_{ki}^{-} \odot  {a}_{ji}^{-} \odot  {a}_{kj}^{-} =
  \\=&a_{ik}^{-} \odot \underbrace{a_{ik}^+ \odot   {a}_{ki}^{-}}_{e} \odot  {a}_{ji}^{-} \odot  {a}_{kj}^{-}  =   {a}_{ik}^{-} \odot {a}_{kj}^{-} \odot {a}_{ji}^{-} 
\end{align*}and
\begin{align*}
 {a}_{ij}^{+} \odot {a}_{jk}^{+} \odot {a}_{ki}^{+} = &  {a}_{ij}^{+} \odot {a}_{jk}^{+} \odot {a}_{ki}^{+}   \odot \underbrace{{a}_{ij}^{-} \odot {a}_{ji}^{+} }_e \odot  \underbrace{{a}_{jk}^{-} \odot {a}_{kj}^{+} }_e=
  \underbrace{  {a}_{ij}^{-} \odot {a}_{ij}^{+} \odot {a}_{jk}^{-} \odot {a}_{jk}^{+} }_{a_{ik}^- \odot a_{ik}^+} \odot
  {a}_{ki}^{+} \odot  {a}_{ji}^{+} \odot  {a}_{kj}^{+} =
  \\=&a_{ik}^- \odot \underbrace{a_{ik}^+ \odot   {a}_{ki}^{+}}_e\odot  {a}_{ji}^{+} \odot  {a}_{kj}^{+}  =  {a}_{ik}^{+} \odot {a}_{kj}^{+} \odot {a}_{ji}^{+};
\end{align*}
thus, by \eqref{eq:Equivalence_odotGConsistency}, $\tilde{A}$ is $[\mathcal{G}]$-consistent.\\
$2 \Rightarrow 3$ It is straightforward.\\
$3 \Rightarrow 2$
As $a_{ii}^-=a_{ii}^+$ and $ a_{ij}^- \odot a_{ji}^+=a_{ij}^+ \odot a_{ji}^-=e$ (Corollary \ref{cor:alternativa3}),  item 2 always holds if three or any two of indices $i,j,k$ are equal. Thus, we consider the case that $i \neq j \neq k$. \\For $i<j<k$,  item 2 is identical to item 3; thus, item 2 holds.\\
Let us consider  $i<k<j$. By item 3, we have  $a_{ij}^{-} \odot a_{ij}^{+} = a_{ik}^{-} \odot a_{ik}^+ \odot a_{kj}^{-} \odot a_{kj}^+$.
Thus, by 	\eqref{composition-1} and $[\mathcal{G}]$-reciprocity (Corollary \ref{cor:alternativa3}), we have:
$$a_{ik}^{-} \odot a_{ik}^+=a_{ij}^{-} \odot a_{ij}^{+} \div (a_{kj}^{-} \odot a_{kj}^+ )= a_{ij}^{-} \odot a_{ij}^{+}  \odot  (a_{kj}^{+} \odot a_{kj}^- )^{(-1)}= a_{ij}^{-} \odot a_{ij}^{+}  \odot a_{jk}^{-} \odot a_{jk}^+.$$
Similarly, we obtain that item 2 holds true for the remaining cases 
$j<i<k$, $j<k<i$, $k<i<j$, $k<j<i$; thus, item 2 is achieved.
\end{proof}
\begin{proof}[Proof of Proposition \ref{prop:implicationsConsistency}]
$1.$  It is straightforward; it is enough to consider the permutation $\sigma$ such that $\sigma(i)=i \; \forall i \in \{1, \ldots, n\}$.
\\
$2.$ If $\tilde{A}$ is approximately $[\mathcal{G}]$-consistent then, by Remark \ref{remark:equivalenceApproximateConsistency}, 
 there is a permutation $\sigma$ such that $\tilde{A}^\sigma$ is Liu's  $[\mathcal{G}]$-consistent and, by Proposition \ref{prop:LiuEquivalent}, we have: $$\tilde{a}^\sigma_{ik}=\tilde{a}^\sigma_{ij} \odot \tilde{a}^\sigma_{jk} \quad \forall i<j<k.$$ 
As a consequence, the following equality holds:
$$[(\tilde{a}^\sigma_{ik})^-, (\tilde{a}^\sigma_{ik})^+ ]=[(\tilde{a}^\sigma_{ij})^- \odot (\tilde{a}^\sigma_{jk})^-, (\tilde{a}^\sigma_{ij})^+ \odot (\tilde{a}^\sigma_{jk})^+ ] \quad \forall i<j<k.$$
Thus, we have that:
$$(\tilde{a}^\sigma_{ik})^- = (\tilde{a}^\sigma_{ij})^- \odot (\tilde{a}^\sigma_{jk})^-, \quad  (\tilde{a}^\sigma_{ik})^+= (\tilde{a}^\sigma_{ij})^+ \odot (\tilde{a}^\sigma_{jk})^+ \quad \forall i<j<k, $$
and finally:
$$(\tilde{a}^\sigma_{ik})^- \odot (\tilde{a}^\sigma_{ik})^+=(\tilde{a}^\sigma_{ij})^- \odot (\tilde{a}^\sigma_{ij})^+ \odot (\tilde{a}^\sigma_{jk})^- \odot (\tilde{a}^\sigma_{jk})^+ \quad \forall i<j<k;$$
thus, by Theorem \ref{Theorem_odotG_consistency_equivalences},  $\tilde{A}^\sigma$ is $[\mathcal{G}]$-consistent, and by Proposition \ref{Prop:invariance_odotG_consistency}, $\tilde{A}$ is $[\mathcal{G}]$-consistent.
\end{proof}
\begin{proof}[Proof of Proposition \ref{Prop:simplification_DG_aijk_aikj}]
By applying $d_{\mathcal{G}}$ in Proposition \ref{dcirc},  equality in \eqref{eq_oru_consistencyIndex_3} can be written as follows:
$$
 d_{[\mathcal{G}]}(\tilde{a}_{ijk},\tilde{a}_{ikj}) = \max \left\{  \tilde{a}_{ijk}^{-} \div \tilde{a}_{ikj}^{-} ,  \; \tilde{a}_{ikj}^{-} \div \tilde{a}_{ijk}^{-},  \;\tilde{a}_{ijk}^{+} \div \tilde{a}_{ikj}^{+} , \; \tilde{a}_{ikj}^{+} \div \tilde{a}_{ijk}^{+}  \right\}.
$$
By applying  $[\mathcal{G}]$-reciprocity, we have:
\[
\tilde{a}_{ijk}^{-} \div \tilde{a}_{ikj}^{-}
=(a^{-}_{ij} \odot a^{-}_{jk} \odot a^{-}_{ki}) \div (a^{-}_{ik} \odot a^{-}_{kj} \odot a^{-}_{ji})= (a^{+}_{ji} \odot a^{+}_{kj} \odot a^{+}_{ik}) \div (a^{+}_{ki} \odot a^{+}_{jk} \odot a^{+}_{ij})
=
\tilde{a}_{ikj}^{+} \div \tilde{a}_{ijk}^{+}
\]
and
\[
\tilde{a}_{ikj}^{-}   \div   \tilde{a}_{ijk}^{-} 
= (a^{-}_{ik} \odot a^{-}_{kj} \odot a^{-}_{ji})   \div             (a^{-}_{ij} \odot a^{-}_{jk} \odot a^{-}_{ki})=   (a^{+}_{ki} \odot a^{+}_{jk} \odot a^{+}_{ij}) \div (a^{+}_{ji} \odot a^{+}_{kj} \odot a^{+}_{ik})
=
\tilde{a}_{ijk}^{+} \div \tilde{a}_{ikj}^{+}.
\]
Thus, the assertion is achieved.
\end{proof}
\begin{proof}[Proof of Proposition \ref{prop:unique_element_for_[G]_consistency}]
By \eqref{eq_oru_consistencyIndex}, \eqref{eq_oru_consistencyIndex_3} and  Proposition \ref{prop:dis_intervals}.
\end{proof}
\begin{proof}[Proof of Corollary \ref{coro:deltaIndetrminacy}]
By $d_\mathcal{G} (a_{ij}^{-}, a_{ij}^{+}) = \max\{\ a_{ij}^{-} \div a_{ij}^{+}, a_{ij}^{+} \div a_{ij}^{-}\}$ and  $a_{ij}^{+}  \geq a_{ij}^{-}$.
\end{proof}
\begin{proof}[Proof of Proposition \ref{prop:unique_element_for_[G]indeterminancy}] By  Definition \ref{def:indeterminancy_index}, Definition \ref{def:indeterminancy_index_3} and Proposition \ref{dcirc}.
\end{proof}

\end{document}